\documentclass[letterpaper]{article}

\usepackage[preprint]{aaai2027}
\usepackage[hyphens]{url}
\usepackage{graphicx}
\usepackage{natbib}
\usepackage{caption}
\usepackage{amsmath,amssymb,amsthm}
\usepackage{booktabs}
\usepackage{cleveref}

\pdfmapfile{+newtx.map}
\pdfmapfile{+qtm.map}
\newtheorem{proposition}{Proposition}
\newtheorem{theorem}{Theorem}

\newtheorem{remark}{Remark}
\newtheorem{assumption}{Assumption}
\crefname{assumption}{Assumption}{Assumptions}
\Crefname{assumption}{Assumption}{Assumptions}

\title{Target-Weighted Neyman Allocation: Experimental Design for Heterogeneous Treatment Effects under Population Shift}
\author{
\begin{tabular}{c@{\hspace{1.2em}}c@{\hspace{1.2em}}c}
Hoang Dang & Luan Pham & Minh Nguyen \\
\small Independent Researcher & \small University of New South Wales, Australia & \small Florida Atlantic University, USA \\
\normalfont\small\texttt{hoangdang112023@gmail.com} &
\normalfont\small\texttt{luan.pham@unsw.edu.au} &
\normalfont\small\texttt{minhnguyen@fau.edu}
\end{tabular}
}
\affiliations{}

\begin{document}

\maketitle

\begin{abstract}

Randomized experiments are often run in one population to guide decisions in
another. Allocating by experimental proportions wastes budget on groups that
rarely appear in deployment, whereas allocating by deployment proportions
under-samples groups that are hard to measure precisely. We propose
\textbf{TWNA} (Target-Weighted Neyman Allocation), a two-stage stratified design that
uses pilot estimates of group--arm outcome variances to allocate final-stage
sample sizes and treatment probabilities for target-weighted group average treatment effect (GATE) precision.
The oracle rule has a closed form and balances deployment importance with
statistical difficulty; the plug-in rule recovers it as pilot variance estimates
stabilize. 
We also extend TWNA to handle uncertainty about deployment composition,
remaining robust whether the target mix is roughly known or entirely unknown.
Finally, we distinguish this weight robustness from a pilot-robust variant for
small or heavy-tailed pilot cells. Simulations and real-covariate
benchmarks show the largest gains when groups are both deployment-important
and difficult to measure.
\end{abstract}

\section{Introduction}
\label{sec:introduction}

Randomized experiments are often run on one population but used to guide decisions for another (\citet{dahabreh2019generalizing, cole2010generalizing}). For example in practice, this means experiments run on engaged users inform policies rolled out to everyone, including segments that
behave differently and are harder to measure (\citet{wang2025enhancing}). A large body of work addresses experimental design (\citet{neyman1992two, cytrynbaum2021optimal}), population generalization (\citet{phan2021designing, egami2023elements, shi2022strategy}), and subgroup estimation (\citet{robertson2024estimating, wei2023efficient}), yet none directly answers a practical question: \textbf{Given that we care about treatment effects for specific groups in a target population, how should we split our sample across those groups and treatment arms?}

This gap has real consequences: allocating by experimental proportions wastes
budget on groups that rarely appear in the deployment population,
while allocating by deployment proportions under-samples groups that are hard
to measure precisely, so neither alone suffices. Industry platforms report the same tension operationally: Netflix finds pre-assignment stratified sampling degrades under production imbalance and instead relies on post-hoc correction (\citet{xie2016improving}), and Facebook's experimentation platform warns that results from one population need not transport to another (\citep{bakshy2014designing}).

We propose \textbf{TWNA}, a two-stage
design for this allocation problem. In the first stage, a small balanced pilot
experiment estimates how variable outcomes are within each group and treatment
arm. In the second stage, the final experiment uses these estimates and target
deployment shares to allocate samples: groups that are more common in
deployment and harder to measure precisely receive more observations, and
within each group, more units are assigned to whichever treatment arm has
higher outcome variance, following the classical Neyman split. Groups with
very low assignment probabilities are clipped to preserve positivity.

When deployment composition is unknown, TWNA still adapts: given a best guess, it targets that guess; given only a worst-case range, it targets the worst case; with no information, it allocates to equalize risk regardless of composition.

Prior work addresses pieces, not the whole. Stratified designs improve precision for the experimental sample, not the deployment population (\citet{zhang2025optimal}). Transportable designs account for the target population, but optimize
overall average effects rather than group-specific ones
(\citet{phan2021designing, shi2022strategy, zhang2025optimal}).
 Post-hoc generalizability methods estimate subgroup effects under population shift, but accept whatever design was already run
(\citet{robertson2024estimating, dahabreh2019generalizing}). Adaptive methods
optimize which policy to deploy, not how accurately each group's effect is estimated (\citet{kasy2021adaptive, wei2025adaptive}).

\subsection{Toy Example}
\label{sec:toy-example}
Consider two groups. Group~1 makes up 80\% of the deployment population
but is easy to measure (outcome standard deviation~$= 1$). Group~2 is
rare in deployment (20\%) but noisy (standard deviation~$= 4$).

Allocating by deployment share puts 80\% of the sample in group~1,
ignoring that group~2 is much harder to estimate precisely.
Allocating by variance puts 80\% in group~2, although errors matter less after deployment. Our rule balances both:
\[
  \rho_1^\star
  =
  \frac{\sqrt{0.8}\cdot 2}{\sqrt{0.8}\cdot 2+\sqrt{0.2}\cdot 8}
  =
  \frac{1}{3},
  \qquad
  \rho_2^\star=\frac{2}{3}.
\]

The numerator $\sqrt{q_k}\,(\sigma_{1k}+\sigma_{0k})$ captures both
forces: deployment importance via $\sqrt{q_k}$ and measurement
difficulty via the total standard deviation.

In simulations and two real-covariate benchmarks, TWNA consistently reduces
estimation error for target-weighted group effects compared to standard
alternatives, and tracks the oracle design once the pilot stabilizes.

\label{sec:contributions}

The paper makes five contributions.
\begin{enumerate}
  \item It defines a principled objective: minimize estimation error for 
  GATEs, weighted by how often each group appears in 
  target population.
  \item It derives a closed-form oracle design specifying how many
  units to sample from each group and how to assign treatment within 
  each group, together with a weight-robust extension for uncertain target
  composition.
  \item It provides a practical two-stage procedure using a small pilot 
  to estimate the inputs needed, and proves this recovers the optimal 
  design as pilot size grows.
  \item It benchmarks the design against alternatives, including
  deployment-proportional, variance-only, and standard Neyman designs,
  across simulations and two datasets.
  \item It quantifies how quickly the plug-in design approaches the oracle as
  the pilot grows, and distinguishes a pilot-robust variance-input variant
  from robustness to deployment-weight uncertainty.
\end{enumerate}

\section{Related Work}
\label{sec:related-work}

This paper sits at the overlap of three literatures: stratified and
adaptive experimental design, design under target-population shift, and
heterogeneous/subgroup effect estimation.

\emph{Stratified and adaptive design.} The classical benchmark is \citet{neyman1992two}, which assigns more observations to noisier strata under stratified sampling. Survey statisticians extend the same Lagrangian logic to \emph{multi-domain survey allocation}: choosing how many units to sample from each subpopulation (domain) so that estimates of several survey variables meet domain-level precision targets, typically by minimizing cost subject to coefficient-of-variation constraints or by optimizing a weighted variance criterion \citet{cochran1977sampling, bethel1989sample, hu2024minimax}. This objective is close in algebraic spirit to ours: both problems allocate more sample to domains that carry more weight in the precision target and are harder to estimate. TWNA differs on two margins. First, multi-domain allocation targets descriptive finite-population quantities, while TWNA targets causal GATEs under a target-weighted loss. Second, multi-domain allocation has no treatment-assignment margin: it chooses only domain sample sizes, while TWNA jointly chooses group sample sizes and within-group treatment probabilities.

In randomized experiments,
\citet{hahn2011adaptive} use pilot data to adapt propensity scores for ATE
estimation, \citet{bai2022optimality} studies matched-pair optimality, and
\citet{tabord2023stratification} uses a first wave to learn stratification trees
and second-wave assignment probabilities. The stratification-tree design
chooses the partition and within-stratum treatment proportions for a given
sample composition---stratum shares remain at population frequencies---so it
optimizes a margin orthogonal to the cross-group allocation studied here; our
experiments instantiate it as within-group refinement and show the two
margins compose.
\citet{cytrynbaum2021optimal} is closest in spirit: he jointly optimizes sampling
and treatment assignment in a two-stage survey experiment, obtaining the same
within-stratum Neyman treatment fraction and a sampling rule proportional to
\(\sigma_1+\sigma_0\) for an ATE objective.\footnote{Describes v2 (2023) of the arXiv preprint. v3 (2026, retitled
``Fine Stratification of Survey Experiments'') generalizes to a cost-weighted
objective and fixes \(p=1/2\) instead.}. Our rule specializes this classical
logic to a different loss, target-weighted MSE for a fixed vector of
GATEs, which is what introduces the factor \(\sqrt{q_k}\) in the cross-group rule.

\emph{Design under target-population shift.}
\citet{phan2021designing} design transportable experiments for a target ATE,
\citet{shi2022strategy} study efficient RCT sampling for a target PATE using
observational data, \citet{egami2023elements} give a external-validity
framework, and \citet{hu2024minimax} study sample selection through minimax welfare
regret. These papers put the target population into design, but their
objectives are ATE, PATE, balance, welfare, or regret. They also act on
different design margins. \citet{phan2021designing} fix $n_1=n_0=n/2$ and
rerandomize the treatment assignment of a \emph{given} sample until
importance-weighted covariates are balanced; with fixed groups as the
covariates, stratified assignment satisfies their balance criterion exactly,
so their procedure complements allocation rules rather than competing with
them. \citet{shi2022strategy} do allocate the RCT sample---proportional to
the target density times the influence-function scale---but hold the
treatment probability fixed; our experiments evaluate their rule and a
variant we adapt to the target-weighted GATE objective. We instead target the
target-weighted MSE of a fixed vector of GATEs, so the optimal design need not
follow experimental population, deployment population, or ATE-optimal
Neyman allocation.

\emph{Heterogeneous and subgroup effect estimation.} A semiparametric literature
develops outcome-model and (augmented) weighting estimators for
subgroup effects \citep{robertson2024estimating}, efficient targeted learning for
pre-specified subgroups \citep{wei2023efficient}, generic machine-learning inference
for heterogeneous effects \citep{chernozhukov2018generic, imai2025statistical}, and
generalized average treatment effects \citep{kallus2019optimal}. These papers inform
estimation after data collection; they do not choose group sample sizes or
treatment probabilities to minimize target-weighted GATE variance before the
final experiment. A related adaptive-design literature targets heterogeneity: \citet{wei2025adaptive} develop response-adaptive experiments for
identifying the subgroup with the largest effect. That objective is
subgroup selection, whereas ours is precision of all pre-specified GATEs under
deployment weights.

\emph{Gap and positioning.} The adjacent work covers only pairwise
combinations of stratified design,
target-population shift, and subgroup estimation. Because the plug-in rule relies
on pilot variance estimates, it inherits the small-pilot concerns emphasized by
\citet{cai2024performance}; robust or conservative versions are therefore important
with small pilots, heavy-tailed outcomes, or nearly homoskedastic arms. The same gap
appears operationally: production platforms correct for imbalance after assignment
(post-hoc variance-reduction corrections, post-stratification) rather than allocating sample ex ante toward
target-weighted precision, and report that pre-assignment stratification
underperforms in practice \citep{xie2016improving}. TWNA instead builds deployment
weighting into the ex-ante allocation rule.

\section{Method}
\label{sec:method}

We study a stratified randomized experiment for estimating GATE under a deployment population that may differ from the
experimental population. Let \(X\in\mathcal X\) denote pre-treatment covariates,
and let \(\mathcal X_1,\ldots,\mathcal X_K\) be fixed, pre-specified groups.
Write \(S=s(X)\in\{1,\ldots,K\}\) for the group label. The target estimands are
\[
  \tau_k
  =
  \mathbb E\{Y(1)-Y(0)\mid S=k\},
  \qquad k=1,\ldots,K.
\]
Let \(q_k=Q(S=k)\) denote the known share of group \(k\) in the deployment
population, with \(q_k>0\) and \(\sum_{k=1}^K q_k=1\).
We consider composition shift only: the group-specific effect transports from
the experimental population \(P\) to deployment, so
\(\mathbb E_Q\{Y(1)-Y(0)\mid S=k\}=\mathbb E_P\{Y(1)-Y(0)\mid S=k\}=\tau_k\).
The design objective is the target-weighted GATE risk
\[
  R_Q(\hat\tau)
  =
  \sum_{k=1}^K q_k
  \mathbb E\{(\hat\tau_k-\tau_k)^2\}.
\]

\subsection{Design Criterion and Oracle Design}
\label{subsec:design-criterion}

For group \(k\), let \(m_k\) denote the final-stage number of experimental units
sampled from that group, and let \(e_k=P(A=1\mid S=k)\) denote the treatment
probability. Define
\[
  V_k(e_k)
  =
  \frac{\sigma_{1k}^2}{e_k}
  +
  \frac{\sigma_{0k}^2}{1-e_k},
\]
where \(\sigma_{ak}^2\) is the arm-\(a\) group-level outcome variance. With
\(M=\sum_k m_k\) and
\(\rho_k=m_k/M\), the leading constant in the target-weighted GATE risk is
\[
  \mathcal B_Q(\rho,e)
  =
  \sum_{k=1}^K
  \frac{q_k}{\rho_k}
  V_k(e_k),
  \qquad
  \rho_k>0,\quad \sum_{k=1}^K\rho_k=1.
\]
Equivalently, \(R_Q(\hat\tau)\approx M^{-1}\mathcal B_Q(\rho,e)\).

\begin{theorem}[Oracle target-weighted design]
\label{thm:proof-oracle}
Suppose \Cref{ass:positivity-variance} holds. The unique minimizer of
  \(\mathcal B_Q(\rho,e)\)
over \(e_k\in[\epsilon,1-\epsilon]\), \(\rho_k>0\), and
\(\sum_k\rho_k=1\) is
\[
  e_k^\star
  =
  \frac{\sigma_{1k}}{\sigma_{1k}+\sigma_{0k}},
  \qquad
  \rho_k^\star
  =
  \frac{\sqrt{q_k}\,(\sigma_{1k}+\sigma_{0k})}
  {\sum_{\ell=1}^K
  \sqrt{q_\ell}\,(\sigma_{1\ell}+\sigma_{0\ell})}.
\]
Equivalently,
\[
  m_k^\star
  \propto
  \sqrt{q_k}\,(\sigma_{1k}+\sigma_{0k}).
\]
\end{theorem}

The argument minimizes \(V_k\) in \(e_k\), which gives
\(V_k(e_k^\star)=s_k^2\) with \(s_k=\sigma_{1k}+\sigma_{0k}\), and then applies
Cauchy--Schwarz to \(\sum_k q_ks_k^2/\rho_k\); equality holds exactly at
\(\rho_k\propto\sqrt{q_k}\,s_k\). Detailed proofs of all results are given in
the supplementary material.

\subsection{Two-Stage TWNA}
\label{subsec:plugin-design}
\paragraph{Stage 1: pilot experiment.}
Run a balanced randomized pilot experiment and estimate the group-level standard
deviations \(\sigma_{1k}\) and \(\sigma_{0k}\). Let the estimates be
\(\hat\sigma_{1k}\) and \(\hat\sigma_{0k}\).

\paragraph{Stage 2: final design.}
Given the pilot estimates, assign treatment in group \(k\) with probability
\(\hat e_k\) and form the raw plug-in proportions \(\tilde\rho_k\),
\[
  \hat e_k
  =
  \Pi_{[\epsilon,1-\epsilon]}
  \left(
  \frac{\hat\sigma_{1k}}
  {\hat\sigma_{1k}+\hat\sigma_{0k}}
  \right),
  \;\;
  \tilde\rho_k
  =
  \frac{
  \sqrt{q_k}\,(\hat\sigma_{1k}+\hat\sigma_{0k})
  }{
  \sum_{\ell}
  \sqrt{q_\ell}\,(\hat\sigma_{1\ell}+\hat\sigma_{0\ell})
  },
\]
where \(\Pi_{[\epsilon,1-\epsilon]}\) denotes clipping to
\([\epsilon,1-\epsilon]\) \citep{hahn2011adaptive}, with the balanced defaults
\(\hat e_k=1/2\) and \(\tilde\rho_k=1/K\) if the corresponding pilot
denominator is zero.
Fix a deterministic allocation floor \(\rho_{\min}\in(0,1/K)\) and let
\(\hat\rho\) be the Euclidean projection of \(\tilde\rho\) onto the
floor-constrained simplex
\(\{\rho:\rho_k\ge\rho_{\min},\ \sum_{k}\rho_k=1\}\), so
\(\hat\rho_k\ge\rho_{\min}\) and the projection is inactive when the raw
plug-in allocation satisfies the floor; set \(\hat m_k\approx M\hat\rho_k\).
Integer allocations are obtained by deterministic rounding while preserving
\(\sum_k\hat m_k=M\) and \(\hat m_k/M-\hat\rho_k=O(M^{-1})\).

\subsection{Variants and Baselines}
\label{subsec:variants}

\paragraph{Robust TWNA.}
To reduce sensitivity to small or heavy-tailed pilot cells, robust TWNA
replaces each pilot standard deviation \(\hat\sigma_{ak}\) throughout the
plug-in rule by a stabilized scale. Let
\(\hat s=1.4826\,\mathrm{MAD}(Y_{\mathrm{pilot}})\) be the global robust
scale, winsorize the outcomes of pilot cell \((a,k)\) at its median
\(\pm\,c\hat s\), and let \(\hat\sigma_{ak,\mathrm{wins}}\) be the resulting
cell standard deviation. With cell size \(n_{ak}\) and shrinkage weight
\(\lambda>0\),
\[
  \tilde\sigma_{ak}
  =
  \Bigl(
    \tfrac{n_{ak}}{n_{ak}+\lambda}\,\hat\sigma_{ak,\mathrm{wins}}^{2}
    +
    \tfrac{\lambda}{n_{ak}+\lambda}\,\hat s^{\,2}
  \Bigr)^{1/2}.
\]
We fix \(c=3\) and \(\lambda=5\) throughout, which leaves near-Gaussian cells
essentially unchanged. Cells with fewer than three observations use \(\hat s\)
directly, and all scale estimates are floored at \(0.05\), which enforces the
bounded-variance input of \Cref{ass:positivity-variance}; the same floor and
small-cell rule apply to the raw \(\hat\sigma_{ak}\) used by TWNA and the other
pilot-based designs.

\paragraph{Learned-partition proxy.}
This baseline is a proxy for covariate-adaptive partition-learning designs:
it partitions the sampling frame into \(K\) quantiles of a
ridge-regularized linear CATE score fitted on the pilot, then applies the
Hahn--Hirano--Karlan (HHK) Neyman rule within the learned strata. The
evaluation target remains the fixed-group GATE vector, although sampling by
learned strata without reweighting adds a composition bias under
within-group effect heterogeneity.

\subsection{Final GATE Estimator}
\label{subsec:gate-estimator}

In the final stage, estimate each group effect by the stratified difference in
means. Let \(\hat\mu_{ak}\) denote the sample mean of outcomes in group \(k\),
arm \(a\in\{0,1\}\), and set
\[
  \hat\tau_k=\hat\mu_{1k}-\hat\mu_{0k}.
\]
For confidence intervals, let \(\hat s_{ak}^2\) be the final-stage sample variance
of outcomes in group \(k\), arm \(a\), with any fixed convention such as zero
when the realized cell has fewer than two observations, and use the interval
\[
  \hat\tau_k
  \pm
  z_{1-\alpha/2}\sqrt{\hat V_k/\hat m_k},
  \qquad
  \hat V_k
  =
  \frac{\hat s_{1k}^2}{\hat e_k}
  +
  \frac{\hat s_{0k}^2}{1-\hat e_k},
\]
the standard groupwise normal interval for a prespecified subgroup effect. This is the fixed-group special case of the sorted-GATEs of \citet{chernozhukov2018generic,imai2025statistical}; because our groups are prespecified rather than learned from an estimated CATE score, the additional cross-unit variance terms that arise from estimated group cutoffs vanish.

\subsection{Theoretical Guarantees}
\label{subsec:main-theory}

The results are for fixed \(K\) and finite-dimensional asymptotics.

\begin{assumption}[Positivity and bounded variance inputs]
\label{ass:positivity-variance}
For some \(0<c_\sigma<C_\sigma<\infty\) and \(0<\epsilon<1/2\),
\[
  \begin{gathered}
    c_\sigma\le\sigma_{ak}\le C_\sigma,
    \qquad
    \frac{\sigma_{1k}}{\sigma_{1k}+\sigma_{0k}}
    \in[\epsilon,1-\epsilon],\\
    0<\rho_{\min}<\min_{k\le K}\rho_k^\star
  \end{gathered}
\]
for \(a\in\{0,1\}\) and \(k=1,\ldots,K\).
\end{assumption}

\begin{assumption}[Pilot variance consistency]
\label{ass:pilot-consistency}
\(\Delta_{n_0}:=\max_{a\in\{0,1\},\,k\le K}|\hat\sigma_{ak}-\sigma_{ak}|
\overset{p}{\to}0\), and \(\Delta_{n_0}=O_p(n_0^{-1/2})\) for rate statements.
\end{assumption}

\begin{theorem}[TWNA regret]
\label{thm:proof-regret}
Suppose \Cref{ass:positivity-variance,ass:pilot-consistency} hold and let
\((\hat\rho,\hat e)\) be TWNA. Then
\[
  \begin{gathered}
  \max_k|\hat e_k-e_k^\star|
  \;\vee\;
  \max_k|\hat\rho_k-\rho_k^\star|
  \overset{p}{\longrightarrow}0,\\
  \bigl|
  \mathcal B_Q(\hat\rho,\hat e)-\mathcal B_Q(\rho^\star,e^\star)
  \bigr|
  \overset{p}{\longrightarrow}0,
  \end{gathered}
\]
and if \(\Delta_{n_0}=O_p(n_0^{-1/2})\), then the conservative regret bound
\(\mathcal B_Q(\hat\rho,\hat e)-\mathcal B_Q(\rho^\star,e^\star)
=O_p(n_0^{-1/2})\) holds, with constants depending on fixed \(K\).
\end{theorem}

\subsection{Robustness to Deployment-Weight Uncertainty}
\label{subsec:weight-robust}

The oracle rule weights each group by \(\sqrt{q_k}\), so it presumes the
deployment weights \(q\) are known. When \(q\) is instead only known to lie in a
set \(\mathcal Q\), the Cauchy--Schwarz argument behind \Cref{thm:proof-oracle}
still resolves the design, and the treatment split is unaffected because it is
separable and \(q\)-free.

\begin{proposition}[Weight-robust target-weighted design]
\label{prop:weight-robust}
Suppose \Cref{ass:positivity-variance} holds and write
\(s_k:=\sigma_{1k}+\sigma_{0k}\ge 2c_\sigma>0\). Write \(\mathcal B_q(\rho,e)=
\sum_k q_k V_k(e_k)/\rho_k\) for the criterion with its dependence on the weight
vector \(q\) made explicit, and let \(\Delta_K^\circ=\{q:q_k>0,\ \sum_k q_k=1\}\).
Let \(\mathcal Q\subseteq\Delta_K\) be nonempty, convex, and compact. The split
\(e_k^\star=\sigma_{1k}/(\sigma_{1k}+\sigma_{0k})\) is optimal in every case below.
\begin{itemize}
\item[(i)] \emph{Average case.} If a distribution on \(\mathcal Q\) has mean
\(\bar q\in\Delta_K^\circ\), the minimizer of \(\mathbb E_q\,\mathcal B_q(\rho,e)=
\mathcal B_{\bar q}(\rho,e)\) is the oracle design of \Cref{thm:proof-oracle} at
\(\bar q\), namely \(\rho_k^{\mathrm{Bayes}}\propto s_k\sqrt{\bar q_k}\). Weight
uncertainty enters only through the mean \(\bar q\).
\item[(ii)] \emph{Worst case.} If \(\mathcal Q\cap\Delta_K^\circ\neq\varnothing\),
then \(q^{\mathrm{lf}}=\arg\max_{q\in\mathcal Q}\sum_k s_k\sqrt{q_k}\) is unique and
interior, and \((\rho^\star,q^{\mathrm{lf}})\) is a saddle point of
\(\min_{\rho,e}\max_{q\in\mathcal Q}\mathcal B_q(\rho,e)\), with value
\(\bigl(\sum_k s_k\sqrt{q_k^{\mathrm{lf}}}\bigr)^2\) and
\(\rho_k^\star\propto s_k\sqrt{q_k^{\mathrm{lf}}}\). The weight-robust design is
thus TWNA at the least-favorable weights.
\item[(iii)] \emph{Full ignorance.} If \(\mathcal Q=\Delta_K\), then
\(q_k^{\mathrm{lf}}\propto s_k^2\) and
\[
  \rho_k^\star\propto(\sigma_{1k}+\sigma_{0k})^2,
\]
with worst-case value \(\sum_k s_k^2\). This design is an equalizer:
\(\mathcal B_q(\rho^\star,e^\star)=\sum_\ell s_\ell^2\) for every \(q\in\Delta_K\).
\end{itemize}
\end{proposition}

Part~(i) shows that averaging over weight uncertainty leaves the design unchanged
up to the mean: the design is TWNA at \(\bar q\), so uncertainty is first-order
irrelevant. Part~(ii) prices worst-case robustness as an allocation
shift toward a least-favorable composition \(q^{\mathrm{lf}}\), and part~(iii)
gives the closed form at the ignorance extreme, where the rule allocates by
\((\sigma_{1k}+\sigma_{0k})^2\) and equalizes target-weighted risk across all
deployment compositions. The three cases share the derivation of
\Cref{thm:proof-oracle}.

Evaluating the exact objective quantifies the trade-off. On the aligned
shift-and-variance scenario of \Cref{tab:simulation}, with the treatment split
fixed at \(e^\star\) and the composition ranging over
\(\mathcal Q_\delta(\hat q)=\{(1-\delta)\hat q+\delta u:u\in\Delta_K\}\),
plug-in TWNA rises from a worst-case value of \(42.87\) at \(\delta=0\) to
\(48.86\) under complete ignorance, whereas the \(s_k^2\) equalizer holds the
constant value \(43.68\) and the minimax rule of part~(ii) traces the lower
envelope between them. Holding the Bayes mean fixed at \(\hat q\) instead,
plug-in TWNA remains exactly optimal and worst-case protection costs
\(1.88\%\). Weight uncertainty therefore has little average-case cost but a
material worst-case cost.

\section{Experiments}
\label{sec:experiments}

\subsection{Setup}
\label{subsec:experiment-methods}

We compare seven core designs, distinguished by the group allocation $m_k$
and the treatment fraction $e_k$; unless a design's treatment fraction is
specified otherwise, it uses $e_k=1/2$. \emph{Uniform balanced} allocates
$m_k=M/K$; \emph{deployment-only} allocates $m_k \propto q_k$;
\emph{variance-only} allocates
$m_k \propto \widehat\sigma_{1k}+\widehat\sigma_{0k}$; and \emph{Uniform
Neyman} keeps equal allocation but uses the Neyman treatment fraction
$\widehat e_k=\widehat\sigma_{1k}/(\widehat\sigma_{1k}+\widehat\sigma_{0k})$.
The \emph{HHK plug-in} design combines the variance-only allocation with the
Neyman fraction; it is an external benchmark that optimizes ATE rather than
target-weighted GATE estimation and ignores $q_k$.\footnote{HHK
\citep{hahn2011adaptive} fixes stratum shares at population frequencies and
chooses only $e_k$; we add the ATE-optimal across-stratum allocation for
parity with the other designs.}
\emph{TWNA} is the two-stage plug-in design of the Method section with
$\epsilon=0.05$ and $\rho_{\min}=0.01$, conservative defaults rather than
tuned optima. The \emph{oracle} applies the same formula, with the same
clipping, allocation floor, and integer-rounding rules, to the true group
and arm standard deviations, as an upper benchmark.

We also instantiate the transportable-design literature as two baselines.
\emph{Shi--Lin (target ATE)} applies the variance-minimizing RCT allocation of
\citet{shi2022strategy} to the fixed groups,
$\widehat m_k \propto q_k\sqrt{\widehat\sigma_{1k}^2+\widehat\sigma_{0k}^2}$,
since their design holds the treatment probability fixed; note the
\emph{linear} dependence on $q_k$, optimal for the target-population ATE.
\emph{Shi--Lin ($\sqrt{q}$)}, our adaptation rather than a rule they propose,
re-solves their fixed-$e$ design problem under our target-weighted GATE loss,
giving
$\widehat m_k \propto \sqrt{q_k}\sqrt{\widehat\sigma_{1k}^2+\widehat\sigma_{0k}^2}$.
It isolates the cross-group margin: relative to TWNA it removes only the
treatment-fraction margin, so the pair decomposes TWNA's gain into its two
components.

In the covariate-rich benchmarks we add \emph{robust TWNA} and the
\emph{learned-partition proxy} of the Method section. We also implement
the stratification trees of \citet{tabord2023stratification} as a
within-group refinement on all three. Each fixed group is refined by a
depth-one tree fitted on the pilot: an exhaustive search over axis-aligned
splits minimizes their empirical variance criterion, which at depth one is
the exact global optimizer, the depth in \(\{0,1\}\) is chosen per group by
their two-fold cross-validation, assignment uses stratified block
randomization at the Neyman proportion within each leaf, and the group
estimate aggregates leaf-level differences in means by realized leaf shares.
Their design fixes sample composition at population frequencies, so
\emph{TM-tree (Uniform)} is the faithful instantiation, while \emph{TM-tree
(TWNA)} applies the identical refinement on TWNA's composition. The pair
separates the within-group margin from the cross-group margin.

The metric is the realized target-weighted squared error
\(\sum_{k=1}^K q_k(\widehat\tau_k-\tau_k)^2\), the Monte Carlo counterpart of
the risk criterion \(R_Q\) defined in the Method section. For each
experiment, relative MSE is normalized by the uniform balanced design under
the same data-generating process and evaluation setting, so values below one
indicate improvement.

\begin{table*}[t]
\centering
\small
\setlength{\tabcolsep}{4pt}
\begin{tabular}{lccc rrrrrr}
\hline
& \multicolumn{3}{c}{Data-generating process}
& \multicolumn{6}{c}{Relative target-weighted GATE MSE} \\
\cmidrule(lr){2-4}\cmidrule(lr){5-10}
Scenario & $q_k$ & $\sigma_{1k}$ & $\sigma_{0k}$
 & Deploy & Var & HHK & SL-$\sqrt{q}$ & TWNA & Oracle \\
\hline
Sanity equal      & $.20$ each          & $\mathbf 1$ & $\mathbf 1$ & 1.000 & 1.001 & 1.003 & 1.007 & 1.006 & 1.004 \\
Deployment shift  & $.38,.25,.18,.12,.07$ & $\mathbf 1$ & $\mathbf 1$ & 1.005 & 1.000 & 0.999 & 0.928 & 0.931 & 0.926 \\
Variance heterog. & $.20$ each          & $\sigma$ & $\sigma$ & 1.001 & 0.851 & 0.855 & 0.851 & 0.856 & 0.851 \\
Aligned shift+var & $.07,.12,.18,.25,.38$ & $\sigma$ & $\sigma$ & 0.701 & 0.730 & 0.732 & 0.689 & 0.693 & 0.689 \\
Anti-aligned      & $.38,.25,.18,.12,.07$ & $\sigma$ & $\sigma$ & 1.596 & 1.069 & 1.075 & 0.985 & 0.986 & 0.982 \\
Arm-var imbalance & $.20$ each          & $2.4,.7,1.8,.8,2.1$ & $.7,2.1,.8,1.9,.9$ & 1.001 & 0.996 & 0.829 & 0.997 & \textbf{0.827} & 0.825 \\
Noisy pilot       & $.07,.12,.18,.25,.38$ & $\sigma$ & $\sigma^{\mathrm{rev}}$ & 0.971 & 0.977 & 0.831 & 0.888 & \textbf{0.764} & 0.762 \\
\hline
\end{tabular}
\caption{Simulation design and results. All scenarios use $K=5$,
$\tau=(.10,.25,.40,.55,.70)$, a pilot drawn with equal group probability
$1/5$, and $1{,}000$ replications; the cell shown is $M=5{,}000$,
$n_0=1{,}000$. Here $\sigma=(.7,.9,1.2,1.7,2.3)$ and $\sigma^{\mathrm{rev}}$
is its reversal. Values are target-weighted GATE MSE relative to uniform
balanced sampling, which is $1.000$ by construction and omitted. TWNA is
within $0.005$ of the oracle in every scenario. The two rows with
$\sigma_{1k}\neq\sigma_{0k}$ isolate the treatment-fraction margin: only
there does TWNA separate from the fixed-$e$ rules. The supplementary material
reports the full $M\in\{500,\dots,5{,}000\}\times n_0\in\{100,\dots,1{,}000\}$
grid. Both margins act in the same direction throughout, but the advantage is
a large-pilot property: at $n_0=100$ the plug-in inputs are noisy enough that
simpler one-signal rules can lead, at the rate quantified in
\Cref{tab:regret}.}
\label{tab:simulation}
\end{table*}

\subsection{Simulation Study}
\label{subsec:simulation-study}

The first-order simulation isolates the allocation mechanism across
combinations of deployment shift and variance heterogeneity. The final GATE
estimator is drawn from its theory-implied normal approximation with variance
\(V_k(e_k)/m_k\), so the experiment targets the first-order risk criterion of
the theory while pilot variances remain noisy, estimated from simulated
balanced pilot data. \Cref{tab:simulation} gives the seven scenarios together
with their results; they span a no-shift equal-variance sanity check,
deployment shift only, variance heterogeneity only, aligned deployment and
variance difficulty, anti-aligned deployment and variance difficulty,
arm-variance imbalance, and a noisy-pilot stress test.

TWNA tracks the oracle to within \(0.005\) in every scenario. When deployment
importance and variance difficulty are aligned, it reduces MSE by 31\% versus
uniform sampling and improves on the ATE-optimal HHK benchmark. In the
anti-aligned scenario the two signals conflict, and TWNA holds the uniform
level while deployment-only sampling, which is indifferent to statistical
difficulty, inflates error by 60\%.

The Shi--Lin pair decomposes the gain into TWNA's two margins. The
\(\sqrt{q_k}\)-weighted cross-group allocation supplies the gain whenever arm
variances are equal; there the two rules agree to within \(0.005\), with
SL-\(\sqrt{q}\) marginally ahead because it estimates no treatment fraction
when the optimum is \(e_k^\star=1/2\). Where arm
variances differ, the Neyman treatment fraction supplies a second margin that
the fixed-\(e\) rules cannot reach: under arm-variance imbalance TWNA attains
\(0.827\) against \(0.997\) for SL-\(\sqrt{q}\), and under the noisy-pilot
stress test \(0.764\) against \(0.888\). Both margins are needed, and each is
active exactly where the theory predicts.

\subsection{End-to-End Validation}
\label{subsec:experiment-practical}

The first-order simulation draws \(\hat\tau_k\) from its asymptotic law. To
confirm that the gain survives realized sampling, we rerun the design on a
joint-signal DGP in which deployment importance and statistical difficulty
point at different groups: \(K=5\), experimental and pilot shares
\(p_k=1/5\), deployment shares \(q=(.50,.20,.10,.10,.10)\), target GATEs
\(\tau=(.10,.20,.30,.40,.50)\), and arm standard deviations
\(\sigma_{1k}=\sigma_{0k}=(.40,.50,1.50,.60,.55)\), so group~1 is
deployment-heavy but easy to measure and group~3 is deployment-light but the
noisiest. Each of \(1{,}000\) replications draws a balanced pilot of
\(n_0=500\) individual outcomes, computes every design's
\((\hat m_k,\hat e_k)\), then draws \(M=2{,}000\) final-stage
\emph{individual} outcomes under realized \(\mathrm{Bernoulli}(\hat e_k)\)
assignment and estimates each GATE by the realized difference in arm means.
TWNA attains relative MSE \(0.911\), against \(0.997\) for HHK, \(1.009\) for
variance-only, and \(1.613\) for deployment-only. The allocation gain is
therefore a property of the design rather than of the first-order
approximation.

\Cref{tab:regret} tracks the plug-in gap of \Cref{thm:proof-regret} directly.
Mean design regret against the oracle allocation falls by 93.5\% as the pilot
grows from \(n_0=100\) to \(n_0=1{,}000\), decaying faster than the
conservative \(O_p(n_0^{-1/2})\) bound, as expected at an interior optimum
where the gradient of \(\mathcal B_Q\) vanishes. A few dozen observations per
group--arm cell capture most of the attainable gain.

\begin{table}[t]
\centering
\small
\begin{tabular}{lrrrr}
\hline
Pilot size \(n_0\) & 100 & 250 & 500 & 1{,}000 \\
\hline
Mean design regret & 2.217 & 0.639 & 0.295 & 0.144 \\
\hline
\end{tabular}
\caption{Plug-in design regret
\(\mathcal B_Q(\hat\rho,\hat e)-\mathcal B_Q(\rho^\star,e^\star)\), averaged
over the seven scenarios of \Cref{tab:simulation} and the budget grid
\(M\in\{500,1{,}000,2{,}000,5{,}000\}\), with \(1{,}000\) replications per
cell.}
\label{tab:regret}
\end{table}

\subsection{Misspecified Deployment Weights}
\label{subsec:sensitivity}

\Cref{prop:weight-robust} handles weight uncertainty when a distribution or an
ambiguity set is available. We next ask what happens when the designer simply
uses the wrong weights and has no such set. On the aligned shift-and-variance
scenario with \(M=2{,}000\), \(n_0=500\), and 500 replications under common
random numbers, the design is built from
\(q^{\mathrm{used}}=(1-\delta)q+\delta\,\mathbf 1/K\), interpolating from exact
knowledge at \(\delta=0\) to a completely uninformative input at \(\delta=1\),
while every design is still evaluated under the true weights \(q\).

\Cref{tab:sensitivity} shows that the degradation is gradual and bounded.
TWNA moves from \(0.682\) to \(0.729\), so even a design input carrying no
information about deployment leaves it \(27\%\) below uniform, because the
variance margin survives intact. The bound is exact rather than empirical: at
\(\delta=1\) the input is uniform, so
\(\rho_k\propto\sqrt{1/K}\,s_k\propto s_k\) and TWNA reduces identically to
the \(q\)-free HHK plug-in. Misspecifying the deployment weights therefore
cannot cost more than discarding them, with equality only at complete
ignorance. Short of that, TWNA is strictly ahead: the paired difference
against HHK is significant at every \(\delta<1\), with \(t\) between
\(-8.0\) and \(-10.4\). The \(\sqrt{q_k}\) margin is worth taking whenever
any credible estimate of the deployment composition exists.

\begin{table}[t]
\centering
\small
\setlength{\tabcolsep}{3.5pt}
\begin{tabular}{lrrrrrr}
\hline
\(\delta\) & 0 & 0.10 & 0.25 & 0.50 & 0.75 & 1.00 \\
\hline
TWNA & 0.682 & 0.683 & 0.687 & 0.696 & 0.709 & 0.729 \\
HHK (\(q\)-free) & 0.729 & 0.729 & 0.729 & 0.729 & 0.729 & 0.729 \\
Oracle & 0.675 & 0.675 & 0.675 & 0.675 & 0.675 & 0.675 \\
\hline
\end{tabular}
\caption{Misspecified deployment weights. The design uses
\(q^{\mathrm{used}}=(1-\delta)q+\delta\,\mathbf 1/K\); all designs are
evaluated under the true \(q\). Values are target-weighted GATE MSE relative
to uniform balanced sampling. HHK and the oracle do not take
\(q^{\mathrm{used}}\) as input and are constant by construction.}
\label{tab:sensitivity}
\end{table}

\subsection{Covariate-Rich Benchmarks}
\label{subsec:covariate-benchmarks}

We next move to covariate distributions rather than group labels alone. The
\emph{calibrated} benchmark draws \(n=700\) units with four covariates,
\(x_1,x_2,x_4\sim N(0,1)\) and \(x_3\sim\mathrm{Bernoulli}(0.25)\), and cuts
the latent risk index
\(\mathrm{risk}=0.6x_1-0.3x_2+0.4x_3-0.2x_4+\eta\), \(\eta\sim N(0,0.3^2)\),
at its empirical quintiles to form \(K=5\) groups. Outcomes are
\(Y(a)=\mu_0(x)+a\tau_k+N(0,\sigma_{ak}^2)\) with
\(\mu_0(x)=0.4x_1+0.2x_2-0.3x_3+0.1x_4+0.3k\),
\(\tau=(.10,.22,.35,.50,.68)\), \(\sigma_0=(.50,.70,1.00,1.50,2.20)\), and
\(\sigma_1=(.60,.80,1.20,1.80,2.60)\). Deployment weights
\(q=(.05,.15,.45,.25,.10)\) concentrate on the medium-to-high-risk groups~3--4
and deliberately not on group~5, which has both the largest effect and the
largest variance, so \(q_k\) and \(\sigma_k\) are only partly aligned. The
realized finite-population moments of one fixed draw are the ground truth for
the oracle and for evaluation; each of 500 replications resamples pilots
(\(n_0=500\)) and final-stage units (\(M=2{,}000\)) from that population.
The two real-covariate benchmarks are IHDP
\citep{shalit2017estimating,louizos2017causal} and LaLonde
\citep{imbens2024lalonde}. Both pair real covariates with simulated potential
outcomes, so they test external relevance rather than constituting an
observational evaluation. IHDP uses the 747-unit covariate matrix with its
simulated potential outcomes, groups are baseline-risk quintiles, and
deployment weights concentrate on high-risk infants. LaLonde uses the
Dehejia--Wahba NSW sample of 185 treated and 260 control units; because only
one outcome is observed there, we simulate potential outcomes from a
covariate-calibrated earnings model, groups are pre-treatment earnings
quintiles, and deployment weights concentrate on low earners.

\Cref{tab:benchmarks} reports all three with paired tests against TWNA, and
the pattern is the same on each. No feasible design improves significantly on
TWNA on any benchmark, TWNA improves significantly on every design that uses
only one of the two signals, and it is statistically indistinguishable from
the oracle throughout. On IHDP it also separates from HHK, \(0.884\) against
\(1.017\), and from its own robust variant; on the calibrated benchmark it
reaches \(0.837\) against the oracle's \(0.843\). The two designs that are
nominally lower than TWNA, SL-\(\sqrt{q}\) and TM-tree (TWNA), sit within
\(1.2\) paired standard errors of it. Both share TWNA's \(\sqrt{q_k}\)
composition and differ only in the treatment-fraction margin, which these
benchmarks leave nearly inactive because within-group arm variances are close
to equal: the mechanism isolated in \Cref{tab:simulation} predicts exactly
this collapse.

The TM-tree rows carry the central structural finding. Holding the
within-group refinement fixed, replacing uniform composition with TWNA
composition reduces relative MSE from \(0.944\) to \(0.802\) (paired
\(t=3.5\)); holding TWNA composition fixed, adding the refinement changes
little (\(t=1.0\)). The gain therefore comes from allocating the budget
across the pre-specified groups, a margin orthogonal to the one that
stratification trees optimize, and the two compose. The learned-partition
proxy is significantly worse than TWNA on all three benchmarks: learning
source-population strata is not itself a substitute for optimizing the
target-weighted fixed-group objective.

\begin{table}[t]
\centering
\small
\setlength{\tabcolsep}{4pt}
\begin{tabular}{lrrr}
\hline
Method & Calibrated & IHDP & LaLonde \\
\hline
Uniform balanced & 1.000$^*$ & 1.000$^*$ & 1.000$^*$ \\
Deployment-only & 1.298$^*$ & 1.123$^*$ & 1.004$^*$ \\
Variance-only & 0.925$^*$ & 1.056$^*$ & 1.005$^*$ \\
Uniform Neyman & 1.017$^*$ & 1.006$^*$ & 1.007$^*$ \\
HHK plug-in & 0.908 & 1.017$^*$ & 0.941 \\
Shi--Lin (target ATE) & 0.907 & 1.026$^*$ & 0.982$^*$ \\
Shi--Lin ($\sqrt{q}$) & 0.791 & 0.947 & 0.829 \\
Learned-partition proxy & 1.049$^*$ & 1.048$^*$ & 1.027$^*$ \\
TM-tree (Uniform) & 0.944$^*$ & 1.000$^*$ & 0.952$^*$ \\
TM-tree (TWNA) & 0.802 & 0.910 & 0.857 \\
TWNA & 0.837 & 0.884 & 0.869 \\
Robust TWNA & 0.848 & 0.971$^*$ & 0.868 \\
Oracle & 0.843 & 0.931 & 0.862 \\
\hline
\end{tabular}
\caption{Covariate-rich benchmarks. Values are target-weighted GATE MSE
relative to uniform balanced sampling, so lower is better. $^*$ marks a
design significantly worse than TWNA in a paired test over the 500 common
replications ($p<0.05$); no design is significantly better than TWNA on any
benchmark. Paired standard errors on the relative differences are
$0.035$--$0.058$.}
\label{tab:benchmarks}
\end{table}

\subsection{Practical Magnitude}
\label{subsec:magnitude}

Because target-weighted risk scales as \(M^{-1}\), a relative MSE of \(r\) is
the precision a uniform design would need a budget of \(M/r\) to match. TWNA's
\(0.693\) in the aligned scenario is therefore worth a \(44\%\) larger
experiment run uniformly, \(0.837\) on the calibrated benchmark is worth
\(20\%\), and \(0.884\) on IHDP is worth \(13\%\). These are savings in
recruitment, not only in a variance criterion.

The same gain appears directly in the reported intervals. Averaged over the
seven scenarios at \(M=5{,}000\) and \(n_0=1{,}000\), the target-weighted mean
width of the groupwise intervals of the Method section falls from \(0.324\)
under uniform allocation to \(0.304\) under TWNA, within \(0.001\) of the
oracle's \(0.303\), while every design holds the nominal \(95\%\) level. The
allocation buys narrower intervals at the same budget rather than trading
coverage for width.

The plug-in rule inherits its inputs from the pilot, so its guarantees are
tied to how well group--arm scales can be estimated there. The supplementary
material reports a battery of operating conditions, including zero-inflated,
rare-event, and contaminated pilots, that maps where the plug-in and robust
variants each apply.

\section{Discussion}
\label{sec:discussion}

Across simulations and covariate-rich benchmarks, allocating by
target-weighted risk reduces GATE error relative to ATE-optimal,
deployment-only, and variance-only designs, with the largest gains when
deployment importance and statistical difficulty coincide. TWNA specializes
Neyman allocation to a different estimand: it keeps the within-group Neyman
treatment fraction but reweights the cross-group rule to
\(\rho_k^\star \propto \sqrt{q_k}\,(\sigma_{1k}+\sigma_{0k})\), so deployment
weight enters sublinearly because precision has diminishing returns in sample
size. The two margins are separately identified: the \(\sqrt{q_k}\) factor
acts whenever deployment composition is informative, and the Neyman fraction
adds a second gain exactly when the two arms differ in variance.

TWNA applies when two signals are credible before the final experiment,
deployment weights and group-by-arm variance inputs, and when the designer
controls recruitment or composition, as in trials, panels, and outbound
campaigns; on always-on platforms with fixed traffic composition
\citep{xie2016improving} it serves as a diagnostic complementing post-hoc
correction. Robust TWNA is the appropriate variant when pilot cells are small
or heavy-tailed.

Three conditions define the scope of the results. The groups are fixed and
pre-specified, so subgroup discovery is a separate problem. The shift is
compositional: group-specific effects are assumed to transport from the
experimental population to deployment, while their relative frequencies
change. Weight uncertainty is handled through a distribution or an ambiguity
set, and \Cref{prop:weight-robust} covers the three cases that arise, using
mean weights in the average case, least-favorable weights in the minimax
case, and the equalizing rule \(\rho_k\propto s_k^2\) under complete
ignorance.

The analysis is batch and optimizes a leading variance criterion; it does not
model recruitment costs, capacity constraints, interference, or multi-wave
adaptation. Relaxing these, as well as learning groups rather than fixing
them in advance, are natural next steps.

\section{Conclusion}
\label{sec:conclusion}

This paper designs two-stage stratified experiments for fixed GATEs when the
deployment population differs in composition from the experimental population.
TWNA jointly chooses group sample sizes and within-group treatment
probabilities through
\(m_k\propto\sqrt{q_k}(\sigma_{1k}+\sigma_{0k})\), with the usual Neyman
treatment fraction within each group. The closed-form oracle makes the
trade-off explicit, and the plug-in design recovers it when pilot
group--arm variance estimates stabilize. When deployment weights are uncertain,
average-case design uses their mean, whereas minimax design uses
least-favorable weights; under complete ignorance the equalizing rule is
\(m_k\propto(\sigma_{1k}+\sigma_{0k})^2\).

The empirical gains are largest when groups are deployment-important and
hard to measure, and the design tracks the oracle once the pilot supplies
stable group--arm scales; the robust variant extends the same rule to small
or heavy-tailed pilot cells. The broader lesson is simple: an experiment
intended to report target-weighted GATE precision should allocate its
sampling budget according to that target-weighted objective.

\section*{Appendix A: Proofs}
\label{app:proofs}

\subsection{Criterion and Proof Assumptions}
\label{app:assumptions}

The deployment weights \(q_k=Q(S=k)\) are known and strictly positive, with
\(\sum_k q_k=1\).

For group \(k\), let \(\sigma_{ak}^2=\operatorname{Var}_P\{Y(a)\mid S=k\}\).
Under treatment probability \(e\), the leading variance input is
\[
  V_k(e)
  =
  \frac{\sigma_{1k}^2}{e}
  +
  \frac{\sigma_{0k}^2}{1-e}.
\]

Let \(M\) be the final-stage budget, let \(m_k\) be the number of final-stage
experimental units assigned to group \(k\), and write
\[
  \rho_k=\frac{m_k}{M},
  \qquad
  \rho_k>0,
  \qquad
  \sum_{k=1}^K\rho_k=1.
\]
The leading design criterion is
\[
  \mathcal B_Q(\rho,e)
  =
  \sum_{k=1}^K
  \frac{q_k}{\rho_k}V_k(e_k).
\]

TWNA uses a balanced pilot experiment to estimate
\(\sigma_{1k}\) and \(\sigma_{0k}\). Given pilot estimates
\(\hat\sigma_{1k}\) and \(\hat\sigma_{0k}\), it sets
\[
  \hat e_k
  =
  \Pi_{[\epsilon,1-\epsilon]}
  \left(
  \frac{\hat\sigma_{1k}}
  {\hat\sigma_{1k}+\hat\sigma_{0k}}
  \right)
\]
with the balanced default \(1/2\) if the pilot denominator is zero. It then
forms the raw plug-in proportions
\[
  \tilde\rho_k
  =
  \frac{
  \sqrt{q_k}\,(\hat\sigma_{1k}+\hat\sigma_{0k})
  }{
  \sum_{\ell=1}^K
  \sqrt{q_\ell}\,(\hat\sigma_{1\ell}+\hat\sigma_{0\ell})
  },
\]
using the balanced default \(1/K\) if the denominator is zero. Fix a
deterministic allocation floor \(\rho_{\min}\in(0,1/K)\), let
\[
  \Delta_{\rho_{\min}}
  =
  \left\{
  \rho\in\mathbb R^K:
  \rho_k\ge\rho_{\min},\ 
  \sum_{k=1}^K\rho_k=1
  \right\},
\]
and set
\[
  \hat\rho
  =
  \Pi_{\Delta_{\rho_{\min}}}(\tilde\rho),
  \qquad
  \hat m_k\approx M\hat\rho_k.
\]
We use the following assumptions.

\paragraph{\Cref{ass:positivity-variance} (Positivity and bounded variance inputs).}
There exist constants \(0<c_\sigma<C_\sigma<\infty\) and
\(0<\epsilon<1/2\) such that
\[
  c_\sigma\le \sigma_{ak}\le C_\sigma,
  \qquad a\in\{0,1\},\quad k=1,\ldots,K,
\]
and
\[
  \frac{\sigma_{1k}}{\sigma_{1k}+\sigma_{0k}}
  \in[\epsilon,1-\epsilon].
\]
The deterministic allocation floor satisfies
\[
  0<\rho_{\min}<\min_{1\le k\le K}\rho_k^\star,
\]
where \(\rho^\star\) is the oracle allocation displayed in
\Cref{thm:proof-oracle} below. Thus the floor is nonbinding at the oracle
allocation.

\paragraph{\Cref{ass:pilot-consistency} (Pilot variance consistency).}
\[
  \begin{aligned}
    \Delta_{n_0}
    &:=\max_{a\in\{0,1\},\,k\le K}
    |\hat\sigma_{ak}-\sigma_{ak}|
    \overset{p}{\longrightarrow}0,\\
    \Delta_{n_0}&=O_p(n_0^{-1/2})
    \quad\text{for rate statements.}
  \end{aligned}
\]
\subsection*{Variant and Baseline Definitions}
\label{app:variant-definitions}

\paragraph{Robust TWNA.}
Let \(Y_{\mathrm{pilot}}\) denote pilot outcomes and define the global
robust scale \(\hat s=1.4826\,\mathrm{MAD}(Y_{\mathrm{pilot}})\), where the
constant \(1.4826\approx 1/\Phi^{-1}(0.75)\) makes \(\hat s\) consistent for
\(\sigma\) under Gaussian outcomes. For each group--arm pilot cell \((a,k)\)
with cell size \(n_{ak}\), winsorize the cell outcomes around their cell
median at radius \(c\hat s\) and let \(\hat\sigma_{ak,\mathrm{wins}}\)
denote the winsorized cell standard deviation. The stabilized pilot scale is
then
\[
  \tilde\sigma_{ak}
  =
  \sqrt{
    \frac{n_{ak}}{n_{ak}+\lambda}\,\hat\sigma_{ak,\mathrm{wins}}^2
    +
    \frac{\lambda}{n_{ak}+\lambda}\,\hat s^2
  },
\]
where \(\lambda > 0\) is a shrinkage weight. Robust TWNA is the same as TWNA
except that every occurrence of \(\hat\sigma_{ak}\) in the plug-in allocation
rule is replaced by \(\tilde\sigma_{ak}\). In all experiments \(c=3\) and
\(\lambda=5\), which leaves nearly Gaussian pilot cells unchanged while
adding moderate shrinkage for small or heavy-tailed cells. In the
implementation, both the global robust scale and the winsorized cell
standard deviations are floored at \(0.05\), and any pilot cell with fewer
than three observations uses the global robust scale directly; the same
floor and small-cell rule apply to the raw pilot estimates
\(\hat\sigma_{ak}\) used by TWNA and the other pilot-based designs,
implementing the bounded-variance-input condition of
\Cref{ass:positivity-variance}.

\paragraph{Learned-partition proxy.}
Using the pilot data, fit a ridge-regularized linear CATE score
\[
  \hat\gamma(x)
  =
  x^\top(\hat\beta_1-\hat\beta_0),
\]
where \(\hat\beta_a\) is the ridge regression of pilot outcomes on
standardized covariates within arm \(a\). Partition the full sampling frame
into \(K\) empirical score quantiles
\(\hat{\mathcal X}_1,\ldots,\hat{\mathcal X}_K\) using \(\hat\gamma(X)\). Let
\(\hat\sigma_{ak}^{\mathrm{LP}}\) be the pilot arm-specific standard deviation
within learned stratum \(k\). The proxy design then applies the HHK Neyman rule
within the learned strata:
\[
  \hat m_k^{\mathrm{LP}}
  \propto
  \hat\sigma_{1k}^{\mathrm{LP}}+\hat\sigma_{0k}^{\mathrm{LP}},
  \qquad
  \hat e_k^{\mathrm{LP}}
  =
  \Pi_{[\epsilon,1-\epsilon]}
  \left(
  \frac{\hat\sigma_{1k}^{\mathrm{LP}}}
  {\hat\sigma_{1k}^{\mathrm{LP}}+\hat\sigma_{0k}^{\mathrm{LP}}}
  \right).
\]
Sampling follows the learned strata, and the fixed-group estimator does not
reweight for original-group \(\times\) learned-stratum inclusion
probabilities, so group estimates carry a composition bias when effects are
heterogeneous within an original group; the fixed-group GATE vector is the
evaluation target, not a claim of unbiasedness.

\paragraph{\Cref{thm:proof-oracle} (Oracle Design).}
\label{app:proof-oracle}

\emph{Statement.} Under \Cref{ass:positivity-variance}, the oracle design
minimizes \(\mathcal B_Q(\rho,e)\) over
\(e_k\in[\epsilon,1-\epsilon]\), \(\rho_k>0\), and \(\sum_k\rho_k=1\). The
unique minimizer is
\[
  e_k^\star
  =
  \frac{\sigma_{1k}}{\sigma_{1k}+\sigma_{0k}},
  \qquad
  \rho_k^\star
  =
  \frac{\sqrt{q_k}\,(\sigma_{1k}+\sigma_{0k})}
  {\sum_{\ell=1}^K
  \sqrt{q_\ell}\,(\sigma_{1\ell}+\sigma_{0\ell})}.
\]

\begin{proof}
Both optimizers are the equality cases of two Cauchy--Schwarz bounds. We use the
Engel form: for \(x_i>0\),
\[
  \left(\sum_i \frac{a_i^2}{x_i}\right)\left(\sum_i x_i\right)
  \ge
  \left(\sum_i a_i\right)^2,
\]
so \(\sum_i a_i^2/x_i\ge(\sum_i a_i)^2/(\sum_i x_i)\), with equality iff
\(a_i/x_i\) is constant.

\emph{Treatment split.} For each \(k\), take \(a=(\sigma_{1k},\sigma_{0k})\) and
\(x=(e,1-e)\), so \(\sum_i x_i=1\) and
\[
  V_k(e)
  =
  \frac{\sigma_{1k}^2}{e}+\frac{\sigma_{0k}^2}{1-e}
  \ge
  (\sigma_{1k}+\sigma_{0k})^2,
\]
with equality iff \(\sigma_{1k}/e=\sigma_{0k}/(1-e)\), i.e.\
\(e=e_k^\star=\sigma_{1k}/(\sigma_{1k}+\sigma_{0k})\), which lies in
\([\epsilon,1-\epsilon]\) by \Cref{ass:positivity-variance}. This holds for every
\(k\) and does not depend on \(\rho\).

\emph{Group sizes.} Write \(s_k=\sigma_{1k}+\sigma_{0k}\) and apply the same bound
with \(a_k=\sqrt{q_k}s_k\), \(x_k=\rho_k\) (so \(\sum_k x_k=1\)):
\[
  \begin{aligned}
  \mathcal B_Q(\rho,e)
  &=
  \sum_k\frac{q_k}{\rho_k}V_k(e_k)
  \ge
  \sum_k\frac{q_k s_k^2}{\rho_k} \\
  &\ge
  \left(\sum_k\sqrt{q_k}s_k\right)^2.
  \end{aligned}
\]
The first inequality is the treatment split (equality iff \(e_k=e_k^\star\)); the
second is Cauchy--Schwarz (equality iff \(\sqrt{q_k}s_k/\rho_k\) is constant).
Both equality conditions are necessary and hold simultaneously at
\((\rho^\star,e^\star)\), which is therefore the unique minimizer, with
\[
  e_k^\star
  =
  \frac{\sigma_{1k}}{\sigma_{1k}+\sigma_{0k}},
  \qquad
  \rho_k^\star
  =
  \frac{\sqrt{q_k}\,(\sigma_{1k}+\sigma_{0k})}
  {\sum_{\ell=1}^K\sqrt{q_\ell}\,(\sigma_{1\ell}+\sigma_{0\ell})},
  \qquad
\]

\[
  V_k(e_k^\star)
  =
  (\sigma_{1k}+\sigma_{0k})^2,
\]
and optimal value
\(\mathcal B_Q(\rho^\star,e^\star)=\left(\sum_k\sqrt{q_k}s_k\right)^2\). This
proves \Cref{thm:proof-oracle}.
\end{proof}

\paragraph{\Cref{thm:proof-regret} (TWNA Regret).}

\emph{Statement.} Suppose
\Cref{ass:positivity-variance,ass:pilot-consistency}
hold, and let \((\hat\rho,\hat e)\) be the pilot-based TWNA. If pilot
consistency (\Cref{ass:pilot-consistency}) holds,
\[
  \Delta_{n_0}
  :=
  \max_{a\in\{0,1\},\,1\le k\le K}
  |\hat\sigma_{ak}-\sigma_{ak}|
  \overset{p}{\to}
  0,
\]
then
\[
  \max_k|\hat e_k-e_k^\star| \overset{p}{\to} 0,
  \qquad
  \max_k|\hat\rho_k-\rho_k^\star| \overset{p}{\to} 0,
\]
and
\[
  \left|
  \mathcal B_Q(\hat\rho,\hat e)
  -
  \mathcal B_Q(\rho^\star,e^\star)
  \right|
  \overset{p}{\to}
  0.
\]
If \(\Delta_{n_0}=O_p(n_0^{-1/2})\), then the conservative bound
\[
  \mathcal B_Q(\hat\rho,\hat e)
  -
  \mathcal B_Q(\rho^\star,e^\star)
  =
  O_p(n_0^{-1/2})
\]
holds, with constants depending on fixed \(K\).

\begin{proof}
The design evaluates the oracle formulas of \Cref{thm:proof-oracle} at the pilot
variances, then clips the split and projects the proportions onto the
floor-constrained simplex. We transfer pilot consistency to the design through
three explicit bounds, each obtained by subtracting two fractions whose
denominators are bounded away from zero. Throughout, \(C<\infty\) denotes a
constant depending only on \(c_\sigma,C_\sigma,\epsilon,K\) and the weights
\(q_k\), possibly changing between displays.

\emph{Reduction to deterministic bounds.} Let
\(E_{n_0}=\{\Delta_{n_0}\le c_\sigma/2\}\). On \(E_{n_0}\) each pilot variance
lies in \([c_\sigma/2,2C_\sigma]\) (\Cref{ass:positivity-variance}), so every
denominator below is bounded away from zero, and
\(P(E_{n_0}^c)\le P(\Delta_{n_0}\ge c_\sigma/2)\to0\). Suppose a random quantity
\(U_{n_0}\) satisfies \(|U_{n_0}|\le C\Delta_{n_0}\) on \(E_{n_0}\); each
\(U_{n_0}\) below is also bounded off \(E_{n_0}\): the split and proportion
differences lie in \([-1,1]\), and the criterion difference is bounded because
both designs lie in the compact set \(D\) defined below. Then for every \(\eta>0\),
\[
  P(|U_{n_0}|\ge\eta)
  \le P(\Delta_{n_0}\ge\eta/C)+P(E_{n_0}^c)\to0,
\]
so \(U_{n_0}\overset{p}{\to}0\); and if \(\Delta_{n_0}=O_p(n_0^{-1/2})\), the same
inequality gives \(U_{n_0}=O_p(n_0^{-1/2})\). It therefore suffices to prove the
deterministic bound \(|U_{n_0}|\le C\Delta_{n_0}\) on \(E_{n_0}\) in each case.

\emph{Treatment split.} With \(g(x,y)=x/(x+y)\) we have
\(e_k^\star=g(\sigma_{1k},\sigma_{0k})\), while the raw pilot split is
\(g(\hat\sigma_{1k},\hat\sigma_{0k})\). Over a common denominator,
\[
  \begin{aligned}
  &|g(\hat\sigma_{1k},\hat\sigma_{0k})-g(\sigma_{1k},\sigma_{0k})| \\
  &\quad=
  \frac{|\hat\sigma_{1k}\sigma_{0k}-\sigma_{1k}\hat\sigma_{0k}|}
       {(\hat\sigma_{1k}+\hat\sigma_{0k})(\sigma_{1k}+\sigma_{0k})}
  \le C\Delta_{n_0},
  \end{aligned}
\]
because adding and subtracting \(\sigma_{1k}\sigma_{0k}\) bounds the numerator by
\(C_\sigma(|\hat\sigma_{1k}-\sigma_{1k}|+|\hat\sigma_{0k}-\sigma_{0k}|)\), while on
\(E_{n_0}\) each denominator factor exceeds \(c_\sigma\). The clip
\(\Pi_{[\epsilon,1-\epsilon]}\) is non-expansive and fixes
\(e_k^\star\in[\epsilon,1-\epsilon]\) (\Cref{ass:positivity-variance}), so
\[
  \max_k|\hat e_k-e_k^\star|
  \le\max_k|g(\hat\sigma_{1k},\hat\sigma_{0k})-g(\sigma_{1k},\sigma_{0k})|
  \le C\Delta_{n_0}.
\]

\emph{Group proportions.} Put \(u_k=\sqrt{q_k}(\sigma_{1k}+\sigma_{0k})\),
\(S=\sum_\ell u_\ell\), with pilot analogues \(\hat u_k,\hat S\), so that
\(\rho_k^\star=u_k/S\) and the raw pilot proportion is
\(\tilde\rho_k=\hat u_k/\hat S\). Each increment is linear in the variance
errors:
\[
  |\hat u_k-u_k|
  \le\sqrt{q_k}\,(|\hat\sigma_{1k}-\sigma_{1k}|+|\hat\sigma_{0k}-\sigma_{0k}|)
  \le C\Delta_{n_0},
\]
\[
  |\hat S-S|\le\sum_\ell|\hat u_\ell-u_\ell|\le C\Delta_{n_0}.
\]
Splitting the difference of the two fractions and using
\(S,\hat S\ge c_\sigma\sum_\ell\sqrt{q_\ell}>0\) on \(E_{n_0}\) and \(u_k\le C\),
\[
  \begin{aligned}
  |\tilde\rho_k-\rho_k^\star|
  &=\left|\frac{\hat u_k}{\hat S}-\frac{u_k}{S}\right| \\
  &\le\frac{|\hat u_k-u_k|}{\hat S}
      +u_k\frac{|\hat S-S|}{\hat S\,S}
  \le C\Delta_{n_0}.
  \end{aligned}
\]
The Euclidean projection \(\Pi_{\Delta_{\rho_{\min}}}\) onto the closed convex set
\(\Delta_{\rho_{\min}}\) is non-expansive in \(\ell_2\) and fixes \(\rho^\star\),
since the floor is nonbinding there (\(\rho_k^\star>\rho_{\min}\),
\Cref{ass:positivity-variance}). It need not be non-expansive in the max norm, so
we pass through \(\ell_2\) via \(\|v\|_\infty\le\|v\|_2\le\sqrt K\,\|v\|_\infty\):
\[
\begin{aligned}
  \max_k|\hat\rho_k-\rho_k^\star|
  &\le\|\hat\rho-\rho^\star\|_2
  \le\|\tilde\rho-\rho^\star\|_2\\
  &\le\sqrt K\max_k|\tilde\rho_k-\rho_k^\star|
  \le C\Delta_{n_0},
\end{aligned}
\]
the fixed factor \(\sqrt K\) being absorbed into \(C\).

\emph{Criterion and regret.} With the true variances fixed,
\[
  \mathcal B_Q(\rho,e)
  =\sum_{k=1}^K\frac{q_k}{\rho_k}
   \left(\frac{\sigma_{1k}^2}{e_k}+\frac{\sigma_{0k}^2}{1-e_k}\right)
\]
is continuously differentiable on the fixed compact convex set
\(D=\{\rho:\rho_k\ge\rho_{\min}/2,\ \sum_k\rho_k=1\}\times[\epsilon,1-\epsilon]^K\),
on which \(\rho_k\ge\rho_{\min}/2>0\) and \(e_k\in[\epsilon,1-\epsilon]\); its
gradient is therefore bounded on \(D\), say by \(C\). The clip and the floor
projection give \((\hat\rho,\hat e),(\rho^\star,e^\star)\in D\). The mean-value
theorem along the segment joining these two points of the convex set \(D\) gives
\[
\begin{gathered}
  \bigl|\mathcal B_Q(\hat\rho,\hat e)
       -\mathcal B_Q(\rho^\star,e^\star)\bigr|\\
  \le C\max_k|\hat\rho_k-\rho_k^\star|
      +C\max_k|\hat e_k-e_k^\star|\\
  \le C\Delta_{n_0}
\end{gathered}
\]
on \(E_{n_0}\). By the reduction this is \(o_p(1)\), and \(O_p(n_0^{-1/2})\) when
\(\Delta_{n_0}=O_p(n_0^{-1/2})\), with constants depending on fixed \(K\).
This proves \Cref{thm:proof-regret}.

\emph{Remark.} The regret bound is conservative. Collect the true variances in
\(\theta=(\sigma_{ak})_{a,k}\), let \(\hat\theta\) be the pilot vector, and for a
variance vector \(\theta'\) write the design maps \(e^\star(\theta')\) and
\(\rho^\star_{\rm fl}(\theta')=\Pi_{\Delta_{\rho_{\min}}}\tilde\rho(\theta')\)
obtained by evaluating the oracle formulas at \(\theta'\). Suppose the optimum is
interior, so the clips and allocation floor are inactive near \(\theta\) and
\(H(\theta')=\mathcal B_Q(\rho^\star_{\rm fl}(\theta'),e^\star(\theta'))\), with
the true variances fixed, is \(C^2\) near \(\theta\). For every \(\theta'\) the
design \((\rho^\star_{\rm fl}(\theta'),e^\star(\theta'))\) is feasible while
\((\rho^\star,e^\star)\) minimizes \(\mathcal B_Q\), so \(H\) is minimized at
\(\theta'=\theta\); hence \(\nabla H(\theta)=0\) and a second-order Taylor
expansion gives
\[
  \begin{aligned}
  \mathcal B_Q(\hat\rho,\hat e)-\mathcal B_Q(\rho^\star,e^\star)
  &=H(\hat\theta)-H(\theta) \\
  &=O_p(\Delta_{n_0}^2)=O_p(n_0^{-1}),
  \end{aligned}
\]
even though the design errors themselves remain \(O_p(n_0^{-1/2})\).
\end{proof}

\paragraph{\Cref{prop:weight-robust} (Weight-robust design).}

\emph{Statement.} Under \Cref{ass:positivity-variance}, with
\(s_k=\sigma_{1k}+\sigma_{0k}\), \(\Delta_K^\circ=\{q:q_k>0,\ \sum_k q_k=1\}\), and
\(\mathcal Q\subseteq\Delta_K\) nonempty, convex, and compact:
(i) if a distribution on \(\mathcal Q\) has mean \(\bar q\in\Delta_K^\circ\), then
\(\mathbb E_q\,\mathcal B_q(\rho,e)\) is minimized at \(e_k^\star\) and
\(\rho_k^{\mathrm{Bayes}}\propto s_k\sqrt{\bar q_k}\);
(ii) if \(\mathcal Q\cap\Delta_K^\circ\neq\varnothing\), then
\(q^{\mathrm{lf}}=\arg\max_{q\in\mathcal Q}\sum_k s_k\sqrt{q_k}\) is unique and
interior and \((\rho^\star,q^{\mathrm{lf}})\) with
\(\rho_k^\star\propto s_k\sqrt{q_k^{\mathrm{lf}}}\) is a saddle point of value
\(\bigl(\sum_k s_k\sqrt{q_k^{\mathrm{lf}}}\bigr)^2\);
(iii) if \(\mathcal Q=\Delta_K\), then \(\rho_k^\star\propto s_k^2\) and
\(\mathcal B_q(\rho^\star,e^\star)=\sum_\ell s_\ell^2\) for every \(q\in\Delta_K\).

\begin{proof}
\emph{Treatment split and reduction.} By the treatment-split step of
\Cref{thm:proof-oracle}, for every \(\rho\) and every \(q\),
\(\mathcal B_q(\rho,e)\ge\sum_k q_k s_k^2/\rho_k\), with equality when
\(e_k=e_k^\star=\sigma_{1k}/(\sigma_{1k}+\sigma_{0k})\) for every \(k\). Since
\(e_k^\star\) is feasible (\Cref{ass:positivity-variance}), \(q\)-free, and
attains equality for every \(q\), it is optimal in (i)--(iii): for any \(e\),
\(\max_{q\in\mathcal Q}\mathcal B_q(\rho,e)\ge\max_{q\in\mathcal Q}
\mathcal B_q(\rho,e^\star)\), and averaging preserves the same ordering. It remains
to study the reduced criterion
\[
  g_q(\rho):=\sum_{k=1}^K\frac{q_k s_k^2}{\rho_k}=\mathcal B_q(\rho,e^\star),
  \qquad \rho_k>0,\ \sum_k\rho_k=1.
\]
For \(q\in\Delta_K^\circ\), the Cauchy--Schwarz (Engel-form) step of
\Cref{thm:proof-oracle} gives the inner-minimization identity
\[
  \min_{\rho}g_q(\rho)=\Big(\sum_k s_k\sqrt{q_k}\Big)^2,
  \]
  \[
  \qquad\text{attained uniquely at }\ \rho_k\propto s_k\sqrt{q_k}.
\]
If \(q\) has a zero coordinate this value is only an infimum, approached as the
corresponding \(\rho_k\downarrow0\); the identity is used below only at interior
points.

\emph{(i) Average case.} For fixed \((\rho,e)\), \(\mathcal B_q(\rho,e)\) is linear
in \(q\), so \(\mathbb E_q\,\mathcal B_q(\rho,e)=\mathcal B_{\bar q}(\rho,e)\) with
\(\bar q=\mathbb E[q]\in\Delta_K^\circ\). Minimizing the right-hand side is
\Cref{thm:proof-oracle} with \(q\) replaced by \(\bar q\), giving \(e_k^\star\) and
\(\rho_k^{\mathrm{Bayes}}\propto s_k\sqrt{\bar q_k}\).

\emph{(ii) Worst case.} Let \(\varphi(q)=\sum_k s_k\sqrt{q_k}\). Because
\(\sqrt{\cdot}\) is strictly concave on \([0,\infty)\) and \(s_k>0\), \(\varphi\) is
strictly concave and continuous on the compact convex set \(\mathcal Q\), so it has
a unique maximizer \(q^{\mathrm{lf}}\).

\emph{Interiority of \(q^{\mathrm{lf}}\).} Fix \(\tilde q\in\mathcal Q\cap
\Delta_K^\circ\). If \(q^{\mathrm{lf}}_j=0\) for some \(j\), set
\(q(t)=(1-t)q^{\mathrm{lf}}+t\tilde q\in\mathcal Q\); its \(j\)-th term contributes
\(s_j\sqrt{t\tilde q_j}=s_j\sqrt{\tilde q_j}\,\sqrt t\), while each other term
changes by \(O(t)\), so \(\varphi(q(t))-\varphi(q^{\mathrm{lf}})\ge
s_j\sqrt{\tilde q_j}\,\sqrt t-Ct>0\) for small \(t>0\), contradicting optimality.
Hence \(q^{\mathrm{lf}}\in\Delta_K^\circ\), so both
\(\nabla\varphi(q^{\mathrm{lf}})_k=s_k/(2\sqrt{q_k^{\mathrm{lf}}})\) and the
inner-minimization identity apply at \(q^{\mathrm{lf}}\).

\emph{Saddle point.} Set \(Z=\varphi(q^{\mathrm{lf}})\) and
\(\rho_k^\star=s_k\sqrt{q_k^{\mathrm{lf}}}/Z\), which is interior and hence
feasible. We verify
\[
  \begin{aligned}
  g_q(\rho^\star)&\le g_{q^{\mathrm{lf}}}(\rho^\star) \\
  &\le g_{q^{\mathrm{lf}}}(\rho)
  \quad\text{for all feasible }\rho\text{ and }q\in\mathcal Q.
  \end{aligned}
\]
The right inequality is the inner-minimization identity at \(q^{\mathrm{lf}}\),
which also gives \(g_{q^{\mathrm{lf}}}(\rho^\star)=Z^2\). For the left inequality,
substitute \(\rho^\star\):
\[
  g_q(\rho^\star)=\sum_k\frac{q_k s_k^2}{\rho_k^\star}
  =Z\sum_k\frac{s_k}{\sqrt{q_k^{\mathrm{lf}}}}\,q_k,
\]
which is linear in \(q\). Since \(q^{\mathrm{lf}}\) maximizes the concave
\(\varphi\) over the convex set \(\mathcal Q\), its first-order optimality condition
\(\sum_k\nabla\varphi(q^{\mathrm{lf}})_k(q_k-q_k^{\mathrm{lf}})\le0\) for all
\(q\in\mathcal Q\) reads \(\sum_k (s_k/\sqrt{q_k^{\mathrm{lf}}})\,q_k\le
\sum_k (s_k/\sqrt{q_k^{\mathrm{lf}}})\,q_k^{\mathrm{lf}}
=\sum_k s_k\sqrt{q_k^{\mathrm{lf}}}=Z\); hence
\(g_q(\rho^\star)\le Z^2=g_{q^{\mathrm{lf}}}(\rho^\star)\). A saddle point gives
\(\min_\rho\max_q g_q(\rho)=\max_q\min_\rho g_q(\rho)=\max_{q\in\mathcal Q}
\varphi(q)^2=Z^2\), the concrete instance of Sion's minimax theorem here.

\emph{(iii) Full ignorance.} Take \(\mathcal Q=\Delta_K\), for which
\(\mathcal Q\cap\Delta_K^\circ\neq\varnothing\) and (ii) applies. By
Cauchy--Schwarz, \(\varphi(q)=\sum_k s_k\sqrt{q_k}\le\sqrt{\sum_k q_k}\,
\sqrt{\sum_k s_k^2}=\sqrt{\sum_k s_k^2}\), with equality iff
\(\sqrt{q_k}\propto s_k\), i.e.\ \(q_k^{\mathrm{lf}}=s_k^2/\sum_\ell s_\ell^2\)
(equivalently, a Lagrange multiplier for \(\sum_k q_k=1\) gives
\(s_k/(2\sqrt{q_k})=\lambda\)). By the inner-minimization identity,
\(\rho_k^\star\propto s_k\sqrt{q_k^{\mathrm{lf}}}\propto s_k^2\). With
\(\rho_k^\star=s_k^2/\sum_\ell s_\ell^2\),
\[
  g_q(\rho^\star)=\sum_k q_k s_k^2\,\frac{\sum_\ell s_\ell^2}{s_k^2}
  =\Big(\sum_\ell s_\ell^2\Big)\sum_k q_k=\sum_\ell s_\ell^2
\]
for every \(q\in\Delta_K\), so \(\rho^\star\) is an equalizer with constant
worst-case value \(\sum_\ell s_\ell^2\).
\end{proof}

\begin{remark}[Price of robustness]
Let \(\mathcal Q=\mathcal Q_\delta(\hat q)\) be a nested family of convex regions
containing a pilot weight estimate \(\hat q\) and shrinking to \(\{\hat q\}\) as
\(\delta\downarrow0\), for instance a confidence region of radius
\(\delta=O_p(N_Q^{-1/2})\) from a target sample of size \(N_Q\). Then
\(q^{\mathrm{lf}}\to\hat q\) and the weight-robust design converges to the plug-in
TWNA at \(\hat q\); the nonnegative gap
\(\varphi(q^{\mathrm{lf}})^2-\varphi(\hat q)^2\) is the price of robustness. When
\(\min_k\hat q_k\ge c_q>0\) and \(\mathcal Q_\delta(\hat q)\) stays within that
interior neighborhood, \(\varphi\) is Lipschitz there and the price is
\(O(\delta)\); near the simplex boundary it can be \(O(\sqrt\delta)\). The family
interpolates between plug-in TWNA (\(\mathcal Q=\{\hat q\}\)) and the ignorance
rule \(\rho_k\propto s_k^2\) (\(\mathcal Q=\Delta_K\)).
\end{remark}

\begin{remark}[Naming]
The ``weight-robust'' design guards against deployment-weight uncertainty and is
distinct from the pilot-robust \emph{Robust TWNA} variant, which stabilizes the
variance inputs \(\sigma_{ak}\) against small or heavy-tailed pilot cells.
\end{remark}

\section*{Appendix B: Experimental Protocols and Supplemental Results}
\label{app:exp-tables}

This appendix records the data-generating parameters, sampling protocols,
and full results for the experiments in the main paper's Experiments section.

\subsection*{Simulation Study: Setup Details}
\label{app:simulation-setup}

All seven scenarios of the Simulation Study share \(K=5\)
groups, the same target GATEs
\(\tau=(0.10,0.25,0.40,0.55,0.70)\), and a group-balanced pilot draw: the
pilot's own group composition is generated with equal probability
\(1/5\) per group regardless of the scenario's deployment weights \(q_k\),
matching the ``balanced pilot'' description in the main text. Scenarios
differ only in the deployment shares \(q_k\) and the arm standard
deviations \(\sigma_{1k},\sigma_{0k}\) used to draw outcomes;
Table~\ref{tab:simulation-scenarios} lists all three vectors for every
scenario. Only \emph{arm-variance imbalance} and \emph{noisy pilot stress}
set \(\sigma_{1k}\neq\sigma_{0k}\), which is what gives the treatment
fraction \(e_k\) a nontrivial oracle value away from \(1/2\) in those two
scenarios; in every other scenario \(e_k^\star=1/2\) exactly.

For each scenario, each cell of the \(M\in\{500,1000,2000,5000\}\times
n_0\in\{100,250,500,1000\}\) grid is replicated \(n_{\mathrm{rep}}=1000\)
times with deterministic child seeds
\(\texttt{child\_seed}(\text{scenario\_id},\text{replication})\) (Appendix~D).
Each replication draws a balanced pilot of size \(n_0\), estimates
\(\hat\sigma_{1k},\hat\sigma_{0k}\) by the per-cell sample standard
deviation floored at \(0.05\) (Appendix~A), forms every design's
\((\hat m_k,\hat e_k)\) via the floor-constrained, largest-remainder
integer allocation of Appendix~A with \(\rho_{\min}=0.01\), and then draws
the final-stage GATE estimator from the first-order approximation
\(\hat\tau_k\sim N(\tau_k, V_k(\hat e_k)/\hat m_k)\); this isolates the allocation mechanism from finite-sample
estimation noise, which the end-to-end experiment (below) checks
separately. The oracle design in each replication uses the same
pilot-estimated \(\hat\sigma\) inputs only to compute the allocation
\emph{regret} diagnostic reported in the main text (design regret is the
gap between the pilot-based and true-variance allocations); its reported
MSE uses the true \(\sigma_{1k},\sigma_{0k}\) directly.

\begin{table*}[t]
\centering
\footnotesize
\caption{Simulation-study scenario definitions. All scenarios share
\(K=5\) groups and target GATEs
\(\tau=(0.10,0.25,0.40,0.55,0.70)\); the pilot is drawn with equal group
probability \(1/5\) in every scenario. Listed are the deployment shares
\(q_k\) and the two arm standard deviations \(\sigma_{1k},\sigma_{0k}\)
used to generate outcomes for group \(k=1,\ldots,5\). Only the bottom two
rows have \(\sigma_{1k}\neq\sigma_{0k}\).}
\label{tab:simulation-scenarios}
\begin{tabular}{lccccc ccccc ccccc}
\hline
& \multicolumn{5}{c}{$q_k$} & \multicolumn{5}{c}{$\sigma_{1k}$} & \multicolumn{5}{c}{$\sigma_{0k}$} \\
Scenario & 1&2&3&4&5 & 1&2&3&4&5 & 1&2&3&4&5 \\
\hline
Sanity equal & .20&.20&.20&.20&.20 & 1.0&1.0&1.0&1.0&1.0 & 1.0&1.0&1.0&1.0&1.0 \\
Deployment shift & .38&.25&.18&.12&.07 & 1.0&1.0&1.0&1.0&1.0 & 1.0&1.0&1.0&1.0&1.0 \\
Variance heterog. & .20&.20&.20&.20&.20 & .7&.9&1.2&1.7&2.3 & .7&.9&1.2&1.7&2.3 \\
Aligned shift+var & .07&.12&.18&.25&.38 & .7&.9&1.2&1.7&2.3 & .7&.9&1.2&1.7&2.3 \\
Anti-aligned & .38&.25&.18&.12&.07 & .7&.9&1.2&1.7&2.3 & .7&.9&1.2&1.7&2.3 \\
Arm-var imbalance & .20&.20&.20&.20&.20 & 2.4&.7&1.8&.8&2.1 & .7&2.1&.8&1.9&.9 \\
Noisy pilot stress & .07&.12&.18&.25&.38 & .7&.9&1.2&1.7&2.3 & 2.3&1.7&1.2&.9&.7 \\
\hline
\end{tabular}
\end{table*}

\subsection*{Core Simulation Results}

Table~\ref{tab:simulation-relative-mse} reports two representative
budget/pilot-size cells; the qualitative ranking is stable across the grid.
Table~\ref{tab:bidirectional} gives the bidirectional GATE-vs-ATE
comparison averaged across all scenarios and budget/pilot-size cells.
The literal HHK design---population stratum shares with the Neyman
treatment fraction---coincides with Uniform Neyman in these simulations
because groups have equal population shares; the HHK plug-in benchmark
therefore adds the ATE-optimal across-stratum allocation, as described in
the main-text setup.

\begin{table*}[t]
\centering
\small
\caption{Simulation relative target-weighted GATE MSE. Values are normalized by
uniform balanced sampling within each cell, so lower is better. With the larger
pilot, TWNA nearly reaches the oracle across scenarios; with the smaller pilot,
it remains competitive when deployment importance and variance conflict, while
one-signal rules can fail sharply in anti-aligned or noisy settings. HHK is the
ATE-optimal Neyman plug-in; U.~Neyman uses uniform group allocation with Neyman
treatment fraction; SL-ATE is the target-ATE allocation of Shi--Lin
($m_k \propto q_k\sqrt{\smash[b]{\sigma_{1k}^2+\sigma_{0k}^2}}$, $e_k=1/2$) and
SL-$\sqrt{q}$ is our adaptation of their fixed-$e$ rule to the target-weighted
GATE loss ($m_k \propto \sqrt{q_k}\sqrt{\smash[b]{\sigma_{1k}^2+\sigma_{0k}^2}}$,
$e_k=1/2$).}
\label{tab:simulation-relative-mse}
\begin{tabular}{lrrrrrrrrr}
\hline
Scenario & Uniform & Deploy & Variance & U.~Neyman & HHK & SL-ATE & SL-$\sqrt{q}$ & TWNA & Oracle \\
\hline
\multicolumn{10}{l}{\textit{$M=500$, $n_0=100$}} \\
Sanity equal & 1.000 & 1.000 & 1.021 & 1.039 & 1.071 & 1.036 & 1.034 & 1.073 & 1.004 \\
Deployment shift & 1.000 & 1.005 & 1.020 & 1.040 & 1.070 & 1.029 & 0.952 & 0.996 & 0.926 \\
Variance heterog. & 1.000 & 1.001 & 0.874 & 1.042 & 0.909 & 0.872 & 0.871 & 0.907 & 0.851 \\
Aligned shift+var & 1.000 & 0.701 & 0.745 & 1.042 & 0.777 & 0.750 & 0.707 & 0.737 & 0.689 \\
Anti-aligned & 1.000 & 1.596 & 1.088 & 1.040 & 1.151 & 1.095 & 1.006 & 1.053 & 0.982 \\
Arm-var imbalance & 1.000 & 1.001 & 1.029 & 0.850 & 0.876 & 1.031 & 1.036 & 0.883 & 0.825 \\
Noisy pilot stress & 1.000 & 0.971 & 1.008 & 0.870 & 0.891 & 0.994 & 0.915 & 0.811 & 0.762 \\
\hline
\multicolumn{10}{l}{\textit{$M=5000$, $n_0=1000$}} \\
Sanity equal & 1.000 & 1.000 & 1.001 & 1.002 & 1.003 & 1.007 & 1.007 & 1.006 & 1.004 \\
Deployment shift & 1.000 & 1.005 & 1.000 & 1.002 & 0.999 & 1.008 & 0.928 & 0.931 & 0.926 \\
Variance heterog. & 1.000 & 1.001 & 0.851 & 1.003 & 0.855 & 0.852 & 0.851 & 0.856 & 0.851 \\
Aligned shift+var & 1.000 & 0.701 & 0.730 & 1.003 & 0.732 & 0.734 & 0.689 & 0.693 & 0.689 \\
Anti-aligned & 1.000 & 1.596 & 1.069 & 1.003 & 1.075 & 1.071 & 0.985 & 0.986 & 0.982 \\
Arm-var imbalance & 1.000 & 1.001 & 0.996 & 0.829 & 0.829 & 0.996 & 0.997 & 0.827 & 0.825 \\
Noisy pilot stress & 1.000 & 0.971 & 0.977 & 0.845 & 0.831 & 0.964 & 0.888 & 0.764 & 0.762 \\
\hline
\end{tabular}
\end{table*}

\subsection*{End-to-End Validation and Design Trade-offs}
\label{app:joint-signal-setup}

\paragraph{End-to-end validation.}

The end-to-end validation uses the joint-signal \(K=5\) scenario---deployment
importance and statistical difficulty pointing at different groups---with
population/pilot shares \(p_k=1/5\),
deployment shares \(q=(0.50,0.20,0.10,0.10,0.10)\), target GATEs
\(\tau=(0.10,0.20,0.30,0.40,0.50)\), and arm standard deviations
\(\sigma_{1k}=\sigma_{0k}=(0.40,0.50,1.50,0.60,0.55)\); group 1 is deployment-heavy but easy, group 3
is deployment-light but the noisiest, and groups 2, 4, 5 are intermediate
comparators. Six designs are compared: uniform balanced, deployment-only,
variance-only, HHK plug-in, TWNA, and oracle.

Each of the \(n_{\mathrm{rep}}=1000\) replications, seeded by
\(\texttt{child\_seed}(81,\text{replication})\), draws a balanced pilot of
\(n_0=500\) individual outcomes (Bernoulli(0.5) treatment, group label drawn
with probability \(p_k\)), estimates \(\hat\sigma_{1k},\hat\sigma_{0k}\) by
the same floored per-cell sample standard deviation as the simulation
study above, computes each design's
\((\hat m_k,\hat e_k)\), and then---unlike the first-order simulation---draws
\(M=2000\) final-stage \emph{individual} outcomes with realized
\(\mathrm{Bernoulli}(\hat e_k)\) treatment assignment rather than sampling
the GATE estimator directly from its asymptotic law. Each group GATE is
estimated by the realized difference in arm means,
\(\hat\tau_k=\bar Y_{1k}-\bar Y_{0k}\). This makes the end-to-end check a
test of realized sampling rather than of the first-order approximation.

\paragraph{Pilot sizing.}

With a fixed total budget of 3000, the pilot-budget frontier in
\Cref{fig:pilot-frontier} shows a broad useful range rather than a universal
optimum: a small pilot leaves scale estimates noisy, while an overly large pilot
reduces final-stage precision. This frontier uses its own \(K=5\) scenario,
distinct from the joint-signal DGP above, with deployment shares
\(q=(0.08,0.12,0.18,0.25,0.37)\) (population/pilot shares uniform), target
GATEs \(\tau=(0.10,0.25,0.40,0.55,0.70)\), and arm standard deviations
\(\sigma_1=(0.75,1.0,1.3,1.8,2.4)\), \(\sigma_0=(0.65,0.9,1.2,1.6,2.2)\); the
total budget of 3000 is split into a pilot of size
\(n_0\in\{50,100,200,350,500,750,1000\}\) and a final stage of
\(M=3000-n_0\), with the same end-to-end sampling and difference-in-means
estimator as above and \(n_{\mathrm{rep}}=1000\) replications per pilot size.

\paragraph{Protocol for the exact weight-robust frontier.}

\Cref{fig:weight-robust-frontier} evaluates all three parts of
\Cref{prop:weight-robust} without Monte Carlo error. It fixes the
\emph{aligned shift-and-variance} scenario of the simulation study,
\(\hat q=(0.07,0.12,0.18,0.25,0.38)\) and
\(\sigma_{1k}=\sigma_{0k}=(0.7,0.9,1.2,1.7,2.3)\), and fixes the treatment
split at \(e_k^\star=1/2\). For
\[
  \mathcal Q_\delta(\hat q)
  =\{(1-\delta)\hat q+\delta u:u\in\Delta_K\},
  \qquad \delta\in[0,1],
\]
the least-favorable composition is computed by the KKT water-filling rule
\(q_k^{\mathrm{lf}}=\max\{(1-\delta)\hat q_k,c_\delta s_k^2\}\), where
\(c_\delta\) makes the weights sum to one. Every curve is then evaluated by
the exact objective \(\mathcal B_q(\rho,e^\star)\).

The left panel maximizes this objective over the \(K\) extreme points of
\(\mathcal Q_\delta(\hat q)\). The robust-minimax curve is the lower envelope,
meeting plug-in TWNA at \(\delta=0\) and the equalizer at \(\delta=1\); the
equalizer is flat because its risk is constant over the whole simplex. For the
right panel, give extreme point
\((1-\delta)\hat q+\delta e_j\) probability \(\hat q_j\). This distribution
has mean \(\hat q\) for every \(\delta\), so its average objective is exactly
\(\mathcal B_{\hat q}(\rho,e^\star)\). Thus it isolates the nominal price of
worst-case protection from a change in the Bayes mean.

\begin{figure*}[t]
\centering
\includegraphics[width=0.92\textwidth]{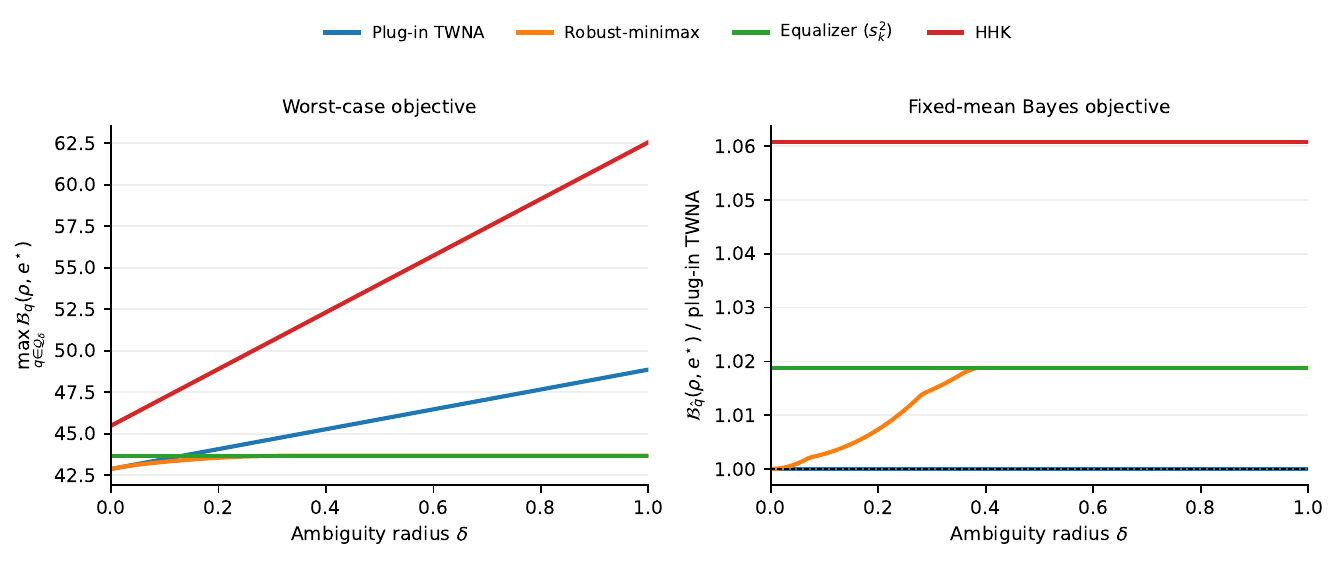}
\caption{Exact weight-robust frontier for \Cref{prop:weight-robust}. Left:
worst-case exact objective over \(\mathcal Q_\delta(\hat q)\); the equalizer is
flat, and robust-minimax connects plug-in TWNA to the equalizer. Right:
fixed-mean Bayes objective relative to plug-in TWNA at \(\hat q\). The plug-in
rule remains Bayes-optimal, while the robust-minimax curve quantifies the
nominal price paid for worst-case protection.}
\label{fig:weight-robust-frontier}
\end{figure*}

\begin{figure}[h]
\centering
\includegraphics[width=0.95\linewidth]{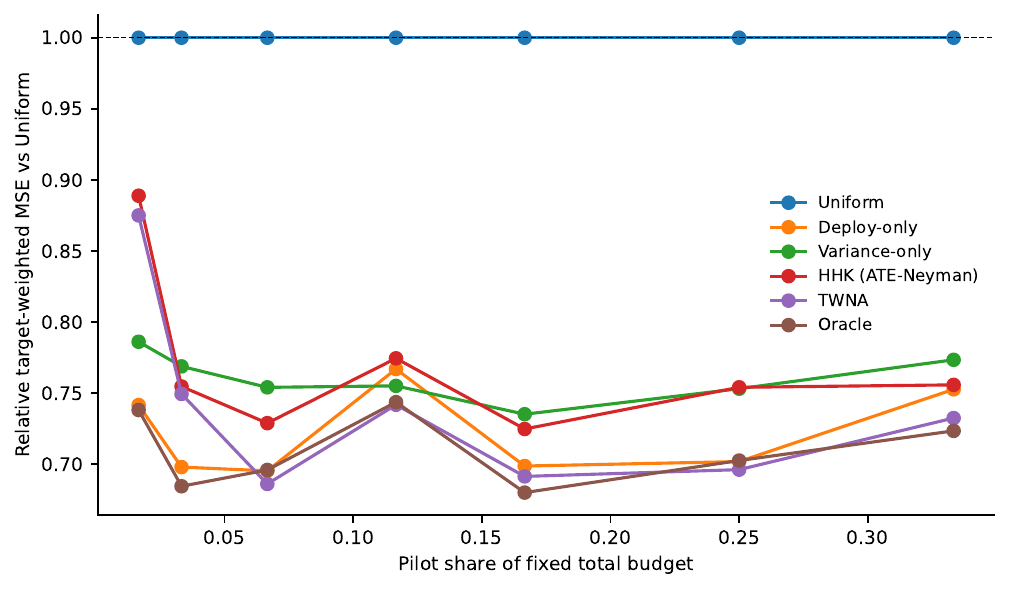}
\caption{Fixed-total-budget pilot frontier. Relative target-weighted MSE is
shown versus uniform allocation, with the pilot share varying under a fixed
total budget of 3,000. The useful region is broad but interior: too little pilot
data leaves scale estimates noisy, while too much pilot spending reduces final
stage precision.}
\label{fig:pilot-frontier}
\end{figure}

\begin{table}[h]
\centering
\small
\begin{tabular}{lrr}
\hline
Method & Rel.\ GATE MSE & Rel.\ ATE MSE \\
\hline
Oracle & 0.863 & 0.934 \\
TWNA & 0.885 & 0.958 \\
SL-$\sqrt{q}$ & 0.914 & 0.990 \\
HHK & 0.923 & 0.911 \\
Variance-only & 0.954 & 0.941 \\
SL-ATE & 0.956 & 1.153 \\
Uniform Neyman & 0.966 & 0.967 \\
Uniform balanced & 1.000 & 1.000 \\
Deployment-only & 1.039 & 1.255 \\
\hline
\end{tabular}
\caption{Bidirectional comparison: target-weighted GATE MSE vs.\ ATE MSE,
averaged across all simulation scenarios and budget/pilot-size cells. Values
are normalized by uniform balanced sampling, so lower is better. TWNA gives the
largest GATE reduction among feasible adaptive rules (0.885) while still
improving ATE MSE (0.958); HHK is better for ATE (0.911) but worse for
target-weighted GATE (0.923), and deployment-only worsens both objectives.
SL-ATE optimizes the \emph{target-population} ATE $\sum_k q_k\tau_k$, which is
neither of the two displayed objectives.}
\label{tab:bidirectional}
\end{table}

\subsection*{Covariate-Rich Benchmarks}

\paragraph{Calibrated synthetic benchmark.}

A calibrated synthetic benchmark uses a four-covariate risk model (three
continuous, one binary) with a variance profile calibrated to infant health
outcomes. Deployment weights focus on medium-to-high risk groups rather than
simply the highest-variance group, creating partial misalignment. With
$M=2{,}000$, $n_0=500$, and 500 Monte Carlo replications,
Table~\ref{tab:calibrated-results} shows TWNA achieves relative MSE 0.837,
outperforming HHK (0.908), deployment-only (1.298), variance-only (0.925),
Uniform Neyman (1.017), and the learned-partition proxy (1.049). Robust TWNA
is close at 0.848; oracle is 0.843.
The Shi--Lin target-ATE allocation (0.907) tracks HHK, while the
SL-$\sqrt{q}$ adaptation reaches 0.791; its paired Monte Carlo difference from
TWNA is $-0.0010$ with standard error $0.0008$, so the two are statistically
indistinguishable here. This is the pattern expected from the design: the
benchmark's arm standard deviations are nearly equal within every group
(ratios $\sigma_{1k}/\sigma_{0k} \le 1.2$), so the treatment-fraction margin
has little to offer and estimating $e_k$ from the pilot only adds noise.

\paragraph{Covariate model.}
The benchmark draws \(n=700\) units with four covariates: continuous
\(x_1,x_2,x_4\sim N(0,1)\) (birth-weight, gestational-age, and
mother-age proxies) and binary \(x_3\sim\mathrm{Bernoulli}(0.25)\) (a
minority-status proxy). A latent risk index
\(\mathrm{risk}=0.6x_1-0.3x_2+0.4x_3-0.2x_4+\eta\), \(\eta\sim N(0,0.3^2)\),
is cut at its empirical quintiles to form the \(K=5\) groups (group~5 is
highest risk). Potential outcomes are
\(Y(0)=\mu_0(x)+N(0,\sigma_{0k}^2)\) and
\(Y(1)=\mu_0(x)+\tau_k+N(0,\sigma_{1k}^2)\) within group \(k\), with
\(\mu_0(x)=0.4x_1+0.2x_2-0.3x_3+0.1x_4+0.3k\), nominal group effects
\(\tau=(0.10,0.22,0.35,0.50,0.68)\), and nominal arm standard deviations
\(\sigma_0=(0.50,0.70,1.00,1.50,2.20)\),
\(\sigma_1=(0.60,0.80,1.20,1.80,2.60)\) increasing with risk. Deployment
weights \(q=(0.05,0.15,0.45,0.25,0.10)\) concentrate on groups~3--4
(medium-to-high risk), deliberately \emph{not} on group~5, which has both
the highest \(\tau_k\) and the highest variance; this is the intended
partial misalignment between \(q_k\) and \(\sigma_k\). The realized
finite-population group moments of this one fixed draw---not the nominal
\(\tau,\sigma\) above---are the ground truth used for the oracle design
and for MSE evaluation, and every Monte Carlo replication
resamples pilots and final-stage individuals (with replacement) from this
same fixed population.

\paragraph{TM-tree implementation details.}
Following the subgroup-targeting extension of
\citet{tabord2023stratification}, which refines \emph{within} pre-specified
groups, each fixed group is refined by a depth-one tree fitted on the pilot:
exhaustive search over axis-aligned splits minimizes their empirical variance
criterion (at depth one this search is the exact global optimizer, so no
evolutionary approximation is involved), the depth in $\{0,1\}$ is chosen per
group by their two-fold cross-validation, assignment uses stratified block
randomization at the Neyman proportion within each leaf, and the group
estimate aggregates leaf-level differences in means by realized leaf shares.

The stratification-tree arms decompose the remaining margins. TM-tree
(Uniform), the faithful instantiation of \citet{tabord2023stratification} in
this setting, improves modestly on uniform composition (0.944; paired
$t=1.2$) through within-group stratification alone. Layering the identical
refinement on TWNA's composition gives 0.802. In paired comparisons, the
composition margin is decisive (TM-tree TWNA vs.\ TM-tree Uniform:
difference $-0.0030$, $t=3.5$), while the refinement margin on top of TWNA
is modest and insignificant ($-0.0008$, $t=1.0$). The depth cross-validation
behaves as their theory suggests: across replications it splits the three
low-noise groups 52--66\% of the time and the two noisiest groups only
12\%, where residual covariate structure is swamped by outcome noise.

\begin{table}[h]
\centering
\small
\begin{tabular}{@{}l@{\hspace{0.45em}}rrr@{}}
\hline
Method & Rel.\ MSE & MSE & SE \\
\hline
Uniform balanced & 1.000 & 0.0210 & 0.0007 \\
Deployment-only & 1.298 & 0.0272 & 0.0011 \\
Variance-only & 0.925 & 0.0194 & 0.0007 \\
Uniform Neyman & 1.017 & 0.0213 & 0.0007 \\
HHK & 0.908 & 0.0191 & 0.0006 \\
Shi--Lin (target ATE) & 0.907 & 0.0190 & 0.0006 \\
Shi--Lin ($\sqrt{q}$) & 0.791 & 0.0166 & 0.0006 \\
\shortstack[l]{Learned-partition\\proxy} & 1.049 & 0.0220 & 0.0007 \\
TM-tree (Uniform) & 0.944 & 0.0198 & 0.0007 \\
TM-tree (TWNA) & 0.802 & 0.0168 & 0.0005 \\
TWNA & 0.837 & 0.0176 & 0.0006 \\
Robust TWNA & 0.848 & 0.0178 & 0.0006 \\
Oracle & 0.843 & 0.0177 & 0.0005 \\
\hline
\end{tabular}
\caption{Calibrated synthetic benchmark. Relative MSE is target-weighted GATE
MSE normalized by uniform balanced sampling, so lower is better; SE is Monte
Carlo standard error over 500 replications. Deployment weights focus on
medium-to-high risk groups, which are not the highest-variance groups. TWNA
improves on HHK (0.837 versus 0.908) and is close to oracle (0.843). The
fixed-\(e\) SL-\(\sqrt{q}\) and
TM-tree (TWNA) variants are slightly lower here because within-group arm
variances are nearly equal, so estimating treatment fractions adds pilot noise.}
\label{tab:calibrated-results}
\end{table}

\paragraph{IHDP and LaLonde: setup details.}
\label{app:ihdp-lalonde-setup}

For both benchmarks of the IHDP and LaLonde section, the \(K=5\)
groups are constructed once from real covariates and held fixed across
\(500\) Monte Carlo replications; only the pilot draw, the resulting
design, and the resampled final-stage individuals vary by replication.
Each replication draws a pilot of size \(n_0\) without replacement from
the finite covariate population, assigns pilot treatment as
\(\mathrm{Bernoulli}(0.5)\), estimates \(\hat\sigma_{1k},\hat\sigma_{0k}\)
(and the robust variant of Appendix~A) from the realized pilot cells,
computes every design's \((\hat m_k,\hat e_k)\), and then samples the
final-stage individuals \emph{with replacement} from the group's finite
population of simulated potential outcomes, assigning realized
\(\mathrm{Bernoulli}(\hat e_k)\) treatment. The estimator is the group-arm
difference in means.
Both benchmarks use \(M=2{,}000\) final-stage units and \(500\)
replications, seeded by \(\texttt{child\_seed}(71,\cdot)\) for IHDP and
\(\texttt{child\_seed}(72,\cdot)\) for LaLonde (Appendix~D).

\paragraph{IHDP/NPCI.} Groups are quintiles of \(-\hat\mu_0(x)\), the
negative of the NPCI response-surface baseline, computed once on the
747-unit IHDP/NPCI covariate matrix (25 covariates), so that group~5 is the
highest-baseline-risk quintile; the pilot size is \(n_0=500\). Deployment
weights are \(q=(0.05,0.10,0.18,0.27,0.40)\), concentrating on the
high-risk quintile. The target GATEs \(\tau_k\) and arm standard
deviations \(\sigma_{1k},\sigma_{0k}\) used by the oracle design and for
evaluation are the finite-population moments of the NPCI simulated
potential outcomes within each quintile, not a further parametric model.
On IHDP the realized target-weighted MSE places TWNA marginally below the
oracle; the raw difference (TWNA minus oracle) is \(-0.00049\) with a 95\%
Monte Carlo interval \([-0.00129,\,0.00030]\) over the 500 replications, which
includes zero and identifies the ordering as finite-sample Monte Carlo
variation rather than a violation of the oracle definition, since the oracle
uses the same clipping, allocation floor, and rounding rules and differs only
in its variance inputs.

\paragraph{LaLonde/Dehejia--Wahba NSW.} The 185 treated and 260 control
units are pooled into one 445-unit covariate population (age, education,
race and marital-status indicators, and 1974/1975 pre-treatment earnings,
log(1+\(\cdot\))-transformed and standardized); groups are quintiles of
\(\log(1+\mathrm{re74})+\log(1+\mathrm{re75})\), with ties broken by rank
so that no quintile is empty, giving pilot size
\(n_0=\min(500,445)=445\). Because only one 1978 earnings outcome
(\(\mathrm{re78}\)) is observed per unit, potential outcomes are simulated
from a linear earnings model fit by least squares,
\(\hat\mu_0(x)=x^\top\hat\beta\) for \(\log(1+\mathrm{re78})\) on the
standardized covariates, with group-specific treatment effects
\(\tau=(0.42,0.34,0.25,0.16,0.08)\) decreasing in pre-treatment earnings
and arm standard deviations \(\sigma_0=(0.80,1.50,1.60,1.00,0.70)\),
\(\sigma_1=(0.90,1.70,1.80,1.10,0.80)\) that peak for low-to-middle
earners rather than only the lowest earners; deployment weights
\(q=(0.40,0.30,0.15,0.10,0.05)\) concentrate on the lowest-earnings
quintile. As with IHDP, the \(\tau_k,\sigma_{1k},\sigma_{0k}\) used by the
oracle design and for evaluation are the realized finite-population group
moments of the simulated \(Y(0),Y(1)\), and every replication resamples
pilots and final-stage individuals (with replacement) from this same
fixed population.

\subsection*{Misspecified Deployment Weights: Protocol}
\label{app:sensitivity-setup}

The misspecification study of the main text fixes the
\emph{aligned shift+variance} scenario, where correct deployment weights
matter most, with \(M=2{,}000\), \(n_0=500\), and \(n_{\mathrm{rep}}=500\)
replications seeded by \(\texttt{child\_seed}(90,\text{replication})\). For
each \(\delta\in\{0,0.10,0.25,0.50,0.75,1.0\}\) the design input is
\(q^{\mathrm{used}}=\Pi\{(1-\delta)q+\delta\,\mathbf 1/K\}\), where \(\Pi\)
renormalizes to the simplex, so \(\delta=0\) is exact knowledge and
\(\delta=1\) discards all information about deployment composition. Only TWNA
consumes \(q^{\mathrm{used}}\): the oracle uses the true \(q\) and the true
\(\sigma_{ak}\), HHK uses no weights at all, and uniform balanced uses no
pilot, so all three are constant in \(\delta\) by construction. Every design
is scored under the true \(q\), so the reported quantity is the cost of
designing for the wrong composition rather than a change of objective. Within
each replication all designs share the same pilot draw and the same
final-stage standard normal shock, which makes the method contrasts paired.
The paired TWNA-minus-HHK differences favor TWNA at every \(\delta<1\), with
\(t\)-statistics \(-7.97,-8.45,-9.02,-9.80,-10.42\) at
\(\delta=0,0.10,0.25,0.50,0.75\) and paired standard errors falling from
\(2.0\times10^{-4}\) to \(6.3\times10^{-5}\) in absolute MSE. At
\(\delta=1\) the difference is identically zero in every replication:
\(q^{\mathrm{used}}=\mathbf 1/K\) gives
\(\rho_k\propto\sqrt{1/K}\,s_k\propto s_k\), which is exactly the HHK plug-in
allocation, so the two designs coincide as functions of the pilot rather than
merely in mean. This makes the bound in the main text an identity: designing
for misspecified weights can never do worse than discarding them.

\subsection*{Interval Coverage and Width}
\label{app:coverage}

The groupwise intervals of the main text are recorded in every replication of
the simulation study. Empirical coverage sits at the nominal level for all
designs and every budget/pilot cell, ranging over \([0.947,0.954]\) across
scenarios, so the allocation rules do not distort the nominal level of the
interval. The target-weighted mean interval width, averaged over the seven
scenarios, separates the designs: at \(M=500\), \(n_0=100\) it is \(1.025\)
for uniform balanced, \(0.986\) for TWNA, and \(0.958\) for the oracle, and at
\(M=5{,}000\), \(n_0=1{,}000\) it is \(0.324\), \(0.304\), and \(0.303\)
respectively. The plug-in design therefore closes almost the whole
uniform-to-oracle gap in reported interval width once the pilot stabilizes.
Because this simulation draws \(\hat\tau_k\) from its first-order normal
approximation, the coverage figures verify the width bookkeeping rather than
a finite-sample normal approximation; the end-to-end experiment provides the
realized-sampling check.

\section*{Appendix C: Robustness, Stress Tests, and Objective Geometry}
\label{app:robustness}

This appendix separates theory-boundary checks, application stress tests, and
mechanism analyses. All simulation-based results use the population-resampling
pipeline of the main experiments.

\subsection*{Heavy-Tailed Outcomes}

The first stress test replaces Gaussian outcomes with a heavy-tailed \(t(3)\)
distribution, stressing the pilot variance estimates. \Cref{fig:appendix-heavy-tail} shows that the main ranking
is largely preserved: deployment-aware rules remain tightly clustered, with
robust TWNA slightly ahead of vanilla TWNA and both remaining
close to the oracle. The takeaway is not that heavy tails overturn the
allocation principle, but that the deployment-targeted design remains
competitive and that pilot regularization can provide a modest stability gain
when pilot cells are noisy.

\begin{figure}[t]
\centering
\includegraphics[width=0.95\linewidth]{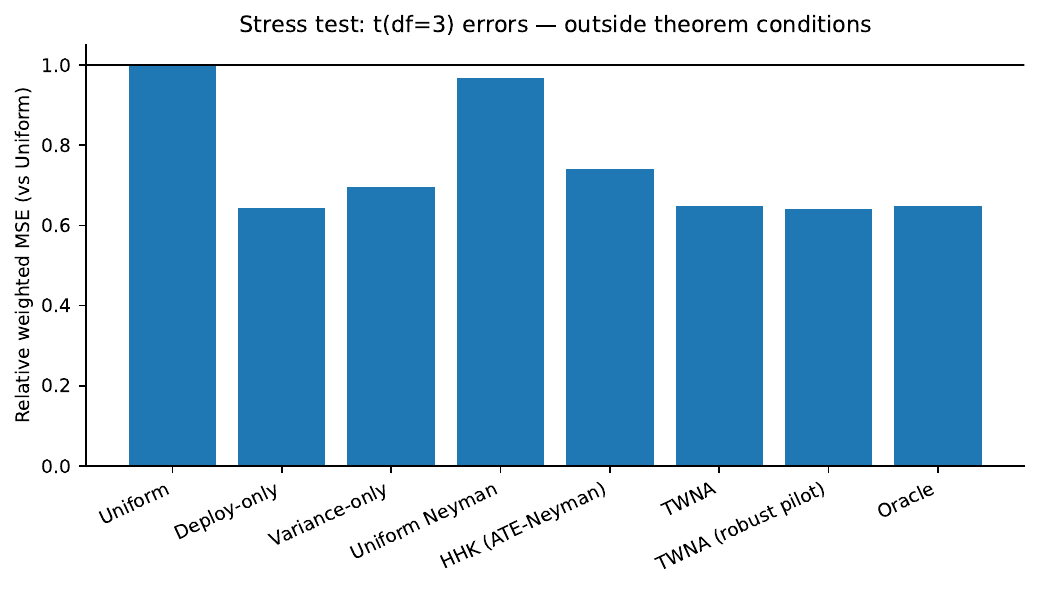}
\caption{Heavy-tailed outcome stress test with \(t(3)\) outcomes. Values are
relative target-weighted GATE MSE normalized by uniform balanced sampling, so
lower is better. The main ranking persists outside the clean theorem conditions:
robust TWNA is slightly ahead of vanilla TWNA, and both remain near the oracle
and well below uniform.}
\label{fig:appendix-heavy-tail}
\end{figure}

\subsection*{Scaling in the Number of Groups}

The second stress test increases the number of fixed groups from \(K=5\) to
\(K=20\) while scaling the final budget as \(M=200K\) so that expected
per-group sample size stays constant. \Cref{fig:appendix-large-k} shows that
TWNA stays below uniform balanced sampling at all three values of
\(K\). The gains are more modest than in the lower-\(K\) benchmarks, but the
pattern does not collapse as the partition grows. For fixed-group deployment
applications, this is the relevant sanity check: when per-group budget is
preserved, deployment-aware allocation continues to deliver a consistent
precision advantage.

\begin{figure}[t]
\centering
\includegraphics[width=0.95\linewidth]{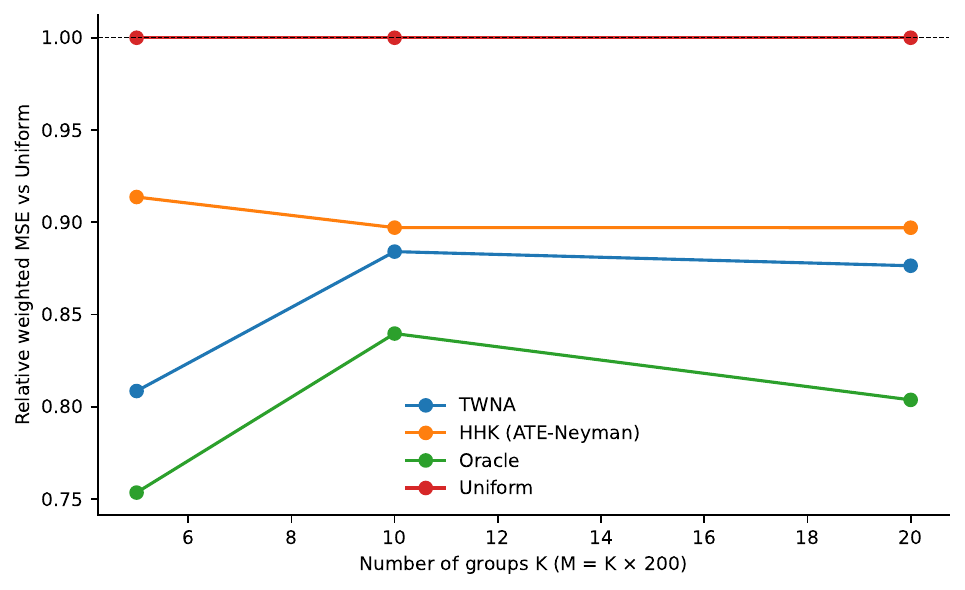}
\caption{Robustness to larger fixed-group partitions. Values are relative
target-weighted GATE MSE normalized by uniform balanced sampling, so lower is
better. The total final-stage budget scales with \(K\) so that expected
per-group sample size remains constant. TWNA remains below uniform for
\(K=5,10,20\), showing that the deployment-aware precision gain persists rather
than collapsing as the partition becomes finer.}
\label{fig:appendix-large-k}
\end{figure}

\subsection*{Application Stress Tests: Construction}

The six application stress cells use the same
population-resampling pipeline as the real-covariate benchmarks: a finite
simulated population with both potential outcomes fixes the estimands (group
effects, arm variances, and group shares are computed empirically from the
population), each replication draws a fresh balanced pilot and a fresh final
sample under each design, and all thirteen designs share the pilot within a
replication. The two deployment profiles per domain share the identical
population and differ only in $q_k$, so aligned/anti-aligned contrasts hold
every population moment fixed. All evaluations use the true $q_k$ of the
profile being run.

\emph{Economic (zero-inflated earnings).} A barrier-to-employment score built
from three covariates defines $K=5$ strata. The outcome is unconditional
annual earnings in thousands of dollars, including zeros for non-employment:
$Y = \mathrm{Bernoulli}(\pi_k)\cdot\mathrm{LogNormal}(\log \mu_k, s_k)$ with
control employment rates $\pi_k$ from 0.86 down to 0.18, median earnings
$\mu_k$ from 28.0 down to 8.5, and log-scale $s_k$ from 0.45 up to 1.95
across barrier strata. Treatment raises employment by 0.02--0.07 and median
earnings by 1.5--5.5\%, both increasing with barrier. The estimand is total
earnings including non-employment; it is not a wage conditional on
employment, so treatment-induced employment selection is out of scope by
construction. The aligned profile puts $q_k=(0.05,0.10,0.15,0.25,0.45)$ on
the highest-barrier strata; the anti-aligned profile reverses it. Sizes:
$M=2{,}000$, $n_0=500$, 500 replications per profile.

\emph{Biostatistics (rare-to-moderate binary endpoint).} Five sites with
enrollment shares $(0.45,0.25,0.16,0.10,0.04)$ enroll individually
randomized patients; the DGP is not cluster-randomized and has no intra-site
correlation. Baseline adverse-event rates rise from 0.03 to 0.25 across
sites and absolute risk reductions from 0.006 to 0.060, so per-unit Bernoulli
variance increases with site risk over this range. The aligned profile
$q_k=(0.05,0.08,0.13,0.24,0.50)$ targets the smallest-enrollment,
highest-risk site; the anti-aligned profile reverses it. The pilot is fixed
at $n_0=100$ (roughly twenty patients per site) in both quick and full runs
because the sparse pilot is itself the stress condition; $M=5{,}000$,
500 replications per profile.

\emph{Online platform (spike-mixture revenue).} Eight segments carry Zipf
traffic shares ($\propto r^{-1.15}$). Revenue per user is zero-inflated:
baseline conversion rates are 0.041--0.055 and treatment multiplies them by
1.02, a two-percent relative lift. Converted revenue in six segments is
moderately dispersed with mean order values \$35--\$70; segment seven is a
stable whale (mean \$450, standard deviation \$28); segment eight is
volatile, mixing a \$45 log-normal base with a 30\% chance of roughly
\$1{,}000 spike orders. The aligned profile weights segments by traffic
share times order value, targeting the stable whale; the anti-aligned
profile weights by squared total standard deviation, targeting the volatile
segment. Sizes: $M=100{,}000$, $n_0=5{,}000$, 150 replications per profile.
This budget is deliberately a low-power, traffic-constrained regime for a
two-percent lift---adequately powered two-arm tests at this baseline require
millions of users---so the cell tests allocation quality under low
signal-to-noise, not a claim of adequate power. None of the three scenarios
models sample-ratio mismatch, instrumentation error, interference,
novelty or ramp-up effects, delayed outcomes, or ratio metrics.

\subsection*{Application Stress Tests: Full Results}

\Cref{tab:domain-stress-full} reports all thirteen designs. The tail-risk
mechanism in the aligned platform cell is documented in
\Cref{tab:domain-tail-risk}: plug-in TWNA and
TM-tree (TWNA) carry across-replication coefficients of variation near 6 and
worst-replication losses two orders of magnitude above the oracle's, while
the robust variant restores both to baseline levels. TM-tree (TWNA) inherits
the plug-in's fragility because it shares the plug-in composition; the
refinement neither causes nor repairs the tail.

\begin{table}[t]
\centering
\footnotesize
\setlength{\tabcolsep}{2.5pt}
\begin{tabular}{lrrrrrr}
\hline
 & \multicolumn{2}{c}{Economic} & \multicolumn{2}{c}{Biostatistics} & \multicolumn{2}{c}{Online platform} \\
Method & align & anti & align & anti & align & anti \\
\hline
Uniform & 1.000 & 1.000 & 1.000 & 1.000 & 1.000 & 1.000 \\
Deploy-only & 0.420 & 2.975 & 0.712 & 1.441 & 0.853 & 0.287 \\
Variance-only & 0.758 & 0.949 & 1.157 & 1.513 & 0.483 & 0.493 \\
Uniform Neyman & 1.923 & 1.191 & 1.525 & 1.551 & 1.462 & 2.126 \\
HHK plug-in & 1.345 & 1.515 & 1.551 & 2.264 & 0.597 & 0.829 \\
SL (target ATE) & 0.411 & 1.725 & 0.997 & 1.631 & 0.746 & 0.381 \\
SL-$\sqrt{q}$ & 0.589 & 1.383 & 0.884 & 1.290 & 0.583 & 0.353 \\
Learned-partition & 0.810 & 1.822 & 5.951 & 2.715 & 2.114 & 2.017 \\
TM-tree (Uniform) & 2.920 & 1.460 & 1.238 & 1.508 & 1.576 & 2.486 \\
TM-tree (TWNA) & 1.107 & 4.007 & 1.193 & 1.821 & 1.923 & 0.934 \\
TWNA & 0.893 & 1.481 & 1.317 & 2.022 & 1.800 & 0.891 \\
Robust TWNA & 0.818 & 2.122 & 0.736 & 1.050 & 0.829 & 0.422 \\
Oracle & 0.300 & 0.695 & 0.649 & 0.949 & 0.383 & 0.240 \\
\hline
\end{tabular}
\caption{Domain-shaped stress tests, all thirteen designs: target-weighted
GATE MSE relative to uniform balanced sampling (lower is better). ``align''
and ``anti'' are the aligned and anti-aligned deployment profiles.}
\label{tab:domain-stress-full}
\end{table}

\begin{table}[t]
\centering
\small
\begin{tabular}{@{}lrr@{}}
\hline
Method & CV of weighted MSE & Worst replication \\
\hline
Oracle & 0.81 & 2.3 \\
Robust TWNA & 0.90 & 3.8 \\
Variance-only & 1.16 & 5.5 \\
Uniform & 1.23 & 10.1 \\
TWNA & 6.02 & 152.2 \\
TM-tree (TWNA) & 6.09 & 164.1 \\
\hline
\end{tabular}
\caption{Tail risk in the aligned platform cell: across-replication
coefficient of variation and worst-replication raw target-weighted MSE.
Plug-in TWNA's failure is a tail event driven by pilots that miss the rare
spike orders; the winsorizing robust estimator removes it.}
\label{tab:domain-tail-risk}
\end{table}

\subsection*{Mechanism Checks: Construction}

This supplementary suite isolates three mechanisms of the allocation rule
under population resampling; each is a mechanism validation, not a prevalence
claim about any domain. The \emph{eligibility rollout} isolates the
cross-group $\sqrt{q}$-composition margin: $K=20$ strata with $q_k$ mass on
four eligible strata and a flat variance profile turn the treatment-fraction
margin off, so TWNA and SL-$\sqrt{q}$ coincide (paired difference zero to four
decimals) while HHK's variance-proportional allocation degrades to $1.045$;
the cell doubles as a large-$K$ scaling check ($M=20{,}000$, $n_0=2{,}000$,
250 replications). \emph{Exacerbation counts} isolate the treatment-fraction
margin: $K=5$ overdispersed counts in which treatment reduces both mean and
variance, so the Neyman fraction assigns more than half of severe strata to
control and TWNA improves on the fixed-$e$ SL-$\sqrt{q}$ design ($M=5{,}000$,
$n_0=1{,}000$, 300 replications). \emph{Contaminated pilot} identifies the
plug-in/robust boundary: the biomarker regime with 1.5\% of pilot outcomes
replaced by recording errors at $\pm 50$ global standard deviations from the
pilot median, applied to the pilot only; a pilot-size sweep re-runs it at
fixed total budget $n_0+M=6{,}000$ with $n_0\in\{50,100,250,500,1000\}$, 200
replications per size, locating where the plug-in overtakes the robust rule.

Four further regimes appear in \Cref{tab:favorable-full} for completeness
without separate narrative: \emph{microenterprise grants} ($K=6$
growth-potential strata) and \emph{biomarker responders} ($K=5$ severity
strata) are additional treatment-fraction cases whose responder-mixture
treated arms inflate treated variance in the $q_k$-targeted strata;
\emph{latency/engagement} ($K=8$ Zipf-traffic segments) is an additional
cross-group case; and \emph{capped revenue} is the aligned platform stress
scenario winsorized at its 99th percentile. Microenterprise and biomarker use
$M=5{,}000$, $n_0=1{,}000$, 300 replications; latency and capped revenue use
$M=20{,}000$, $n_0=5{,}000$, 150 replications.

\subsection*{Mechanism Checks: Full Results}

\Cref{tab:favorable-full} reports all thirteen designs.
TM-tree (TWNA) is best or ties best in four of seven regimes, consistent
with the calibrated-benchmark finding that the within-group refinement
composes with TWNA's cross-group composition; it fails alongside plug-in
TWNA under contamination (1.313). The two cells in which TWNA's mean dips
below the oracle's (microenterprise, eligibility) are Monte Carlo artifacts,
the paired TWNA-minus-oracle differences including zero. In the pilot-size
sweep (\Cref{tab:favorable-sweep}) robust TWNA leads only at $n_0=50$; from
$n_0=250$ the plug-in leads, and both approach the oracle as $n_0$ grows.

\begin{table*}[t]
\centering
\small
\setlength{\tabcolsep}{3.5pt}
\begin{tabular}{lrrrrrrr}
\hline
Method & Micro & Elig. & Biom. & Counts & Latency & Capped & Contam. \\
\hline
Uniform & 1.000 & 1.000 & 1.000 & 1.000 & 1.000 & 1.000 & 1.000 \\
Deploy-only & 0.481 & 0.985 & 0.602 & 0.497 & 0.795 & 0.955 & 0.666 \\
Variance-only & 0.595 & 0.968 & 0.697 & 0.565 & 0.800 & 0.893 & 1.402 \\
Uniform Neyman & 0.713 & 0.997 & 0.753 & 0.770 & 0.929 & 1.101 & 1.307 \\
HHK plug-in & 0.465 & 1.045 & 0.612 & 0.457 & 0.707 & 0.963 & 1.880 \\
SL (target ATE) & 0.549 & 0.971 & 0.664 & 0.510 & 0.765 & 0.985 & 0.772 \\
SL-$\sqrt{q}$ & 0.514 & 0.609 & 0.650 & 0.511 & 0.721 & 0.848 & 0.915 \\
Learned-partition & 0.590 & 0.991 & 1.037 & 0.779 & 1.493 & 1.941 & 2.339 \\
TM-tree (Uniform) & 0.700 & 0.967 & 0.822 & 0.736 & 0.903 & 1.231 & 1.587 \\
TM-tree (TWNA) & 0.390 & 0.591 & 0.563 & 0.392 & 0.619 & 1.024 & 1.313 \\
TWNA & 0.387 & 0.598 & 0.543 & 0.442 & 0.668 & 0.937 & 1.256 \\
Robust TWNA & 0.432 & 0.621 & 0.546 & 0.453 & 0.721 & 0.857 & 0.553 \\
Oracle & 0.399 & 0.625 & 0.527 & 0.397 & 0.661 & 0.784 & 0.516 \\
\hline
\end{tabular}
\caption{Supplementary mechanism checks, all thirteen designs:
target-weighted GATE MSE relative to uniform balanced sampling (lower is
better). The seven columns are mechanism regimes: microenterprise grants
(Micro), eligibility rollout (Elig.), biomarker responders (Biom.),
exacerbation counts (Counts), latency/engagement (Latency), capped revenue
(Capped), and contaminated pilot (Contam.).}
\label{tab:favorable-full}
\end{table*}

\begin{table}[t]
\centering
\small
\begin{tabular}{rrrrr}
\hline
$n_0$ & TWNA & Robust TWNA & SL-$\sqrt{q}$ & Oracle \\
\hline
50 & 0.713 & 0.584 & 0.707 & 0.476 \\
100 & 0.614 & 0.613 & 0.615 & 0.477 \\
250 & 0.523 & 0.654 & 0.666 & 0.510 \\
500 & 0.549 & 0.545 & 0.619 & 0.538 \\
1000 & 0.448 & 0.508 & 0.555 & 0.431 \\
\hline
\end{tabular}
\caption{Pilot-size sweep on the biomarker regime at fixed total budget
$n_0+M=6{,}000$: relative target-weighted GATE MSE versus uniform. Robust
TWNA leads only at the smallest pilot; the plug-in leads from $n_0=250$.}
\label{tab:favorable-sweep}
\end{table}

\subsection*{Exact Objective Geometry}

The phase diagram of \Cref{fig:favorable-phase} is an exact calculation on
the population objective, not a Monte Carlo experiment. With $K=5$, uniform
$p_k$, control standard deviations $(0.85, 1.00, 1.30, 1.75, 2.30)$, and
$M=100{,}000$ with the allocation floor removed, deployment weights
interpolate from uniform to $(0.03, 0.06, 0.12, 0.27, 0.52)$ and the two
highest-$q$ strata scale their treated-arm standard deviation by a ratio
ranging from 1 to 4. Every design is evaluated at its exact objective value
under the true variances; the best feasible baseline is the minimum over
uniform, deploy-only, variance-only, HHK, and SL-$\sqrt{q}$. TWNA never
loses on this map by construction---it is the objective's minimizer under
true inputs---so the map's content is the size of the gain, not its sign.
The gain over the best feasible baseline peaks near 12\% at arm ratio 2.5
and falls at higher ratios because HHK shares the Neyman treatment fraction
and closes the gap as the treatment-fraction margin grows; the gain over
fixed-$e$ designs rises monotonically to 24\% at ratio 4.

\begin{figure}[t]
\centering
\includegraphics[width=0.95\linewidth]{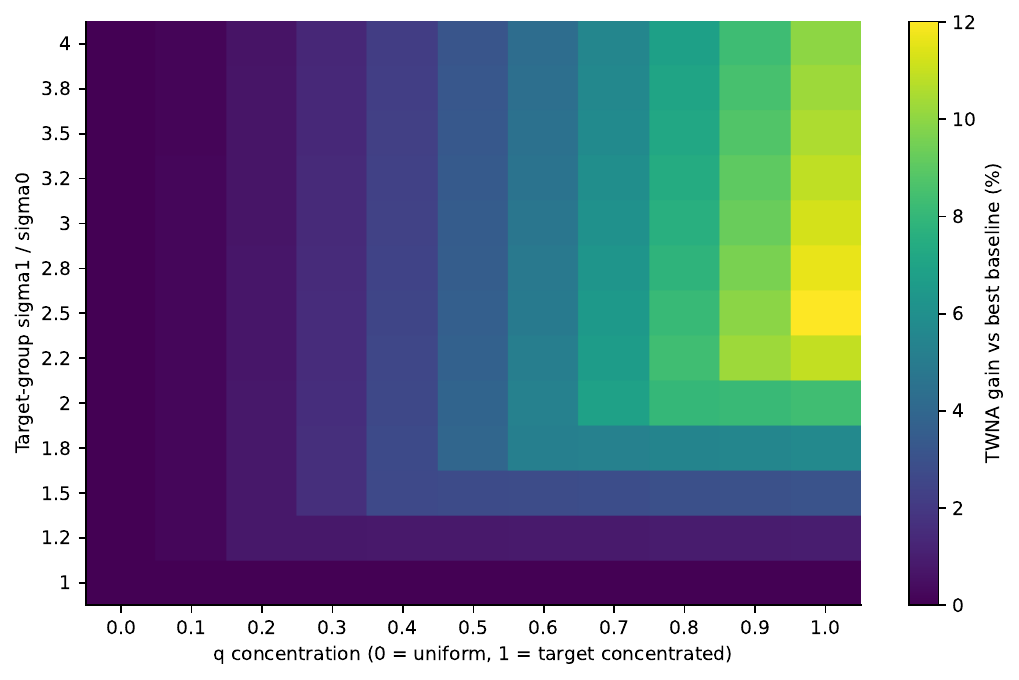}
\caption{Exact objective geometry over deployment concentration and
within-group arm-variance imbalance, computed from true variances with the
allocation floor removed. Left: TWNA's guaranteed gain over the best feasible
baseline peaks near 12\%. Right: the gain over fixed-$e$ designs
(SL-$\sqrt{q}$) grows with arm imbalance, reaching 24\%.}
\label{fig:favorable-phase}
\end{figure}

\section*{Appendix D: Reproducibility Details}

The empirical pipeline is organized as a small code appendix in the repository.
File \texttt{experiment/code/run\_all.py} reproduces the full tables and
figures, \texttt{config.py} fixes the main run configuration, and
\texttt{plotting.py} writes the paper figures and tables. The design rules are
implemented in \texttt{designs.py}; estimation and evaluation are implemented in
\texttt{estimators.py} and \texttt{metrics.py}; the simulation, benchmark, and
appendix drivers are \texttt{simulation.py}, \texttt{real\_data.py},
\texttt{sensitivity.py}, \texttt{weight\_robust.py}, and \texttt{appendix.py}.
The exact weight-robust frontier uses no random-number generation. All data loading and
pre-processing used in the paper are contained in \texttt{dgp.py} and
\texttt{real\_data.py}, including group construction, pilot sampling, and the
IHDP/NPCI and LaLonde benchmark transformations.

The default full run uses the fixed configuration in
\path{experiment/results/final_result_for_paper/configs/run_config.json}: global seed 20260522,
\(K=5\), \(M\in\{500,1000,2000,5000\}\),
\(n_0\in\{100,250,500,1000\}\), \(n_{\mathrm{rep}}=1000\),
\(\epsilon=0.05\), and \(\alpha=0.05\). For experiment (i) and replication
(r), deterministic child seeds are generated by
\[
  \begin{aligned}
  s_{i,r} &:= \texttt{child\_seed}(i,r) \\
          &= 20260522 + 100000\cdot i + r.
  \end{aligned}
\]
The appendix stress test for heavy tails fixes \(df=3\), and the scaling study
below varies only the number of fixed groups \(K\). The experiments were run on
a MacBook Air M2 with 16GB RAM running macOS 15.7.3 and Python 3.13.9, using
NumPy 2.3.5, pandas 2.3.3, matplotlib 3.10.8, and scikit-learn 1.6.1.

Every experiment turns the raw plug-in proportions of Appendix~A into
integers with the same deterministic rule: proportions are first projected
onto the floor-constrained simplex with \(\rho_{\min}=0.01\)
(\texttt{floor\_constrained\_proportions} in \texttt{designs.py}), then
\(M\hat\rho_k\) is floored and the \(M-\sum_k\lfloor M\hat\rho_k\rfloor\)
remaining units are given, one each, to the groups with the largest
fractional remainders (the standard largest-remainder / Hamilton
apportionment rule), so that \(\sum_k\hat m_k=M\) exactly in every
replication.


\end{document}